\documentclass{article}

\usepackage{microtype}
\usepackage{graphicx}
\usepackage{booktabs} % for professional tables
\usepackage{multicol}
\usepackage{algorithm,algpseudocode}
\usepackage{amssymb}
\usepackage{mathtools}
\usepackage{soul}
\usepackage{booktabs}
\usepackage{siunitx}
\usepackage{multirow}
\usepackage{xcolor}
\usepackage{enumitem}
\usepackage{lipsum}
\usepackage{breqn}
\usepackage{amsthm}

\usepackage{booktabs}
\usepackage{siunitx}
\usepackage{multirow}
\usepackage{graphicx}
\usepackage{lipsum}
\usepackage{enumitem}
\usepackage{multicol}
\usepackage{amsthm}
\usepackage{amsmath,amssymb,mathtools}
\usepackage{booktabs}
\usepackage{siunitx}
\usepackage{multirow}
\usepackage{xcolor}
\usepackage{soul}

\usepackage{breqn}
\usepackage{bbm}
\usepackage{empheq}
\usepackage{wrapfig, blindtext}
\usepackage{multicol}
\usepackage{caption}
\usepackage{graphicx}
\usepackage{hyperref}

\usepackage{bbm}
\usepackage{xcolor}

\newcommand{\reb}[1]{{ #1}}

\usepackage[utf8]{inputenc} % allow utf-8 input
\usepackage[T1]{fontenc}    % use 8-bit T1 fonts
\usepackage{hyperref}       % hyperlinks
\usepackage{url}            % simple URL typesetting
\usepackage{booktabs}       % professional-quality tables
\usepackage{amsfonts}       % blackboard math symbols
\usepackage{nicefrac}       % compact symbols for 1/2, etc.
\usepackage{microtype}      % microtypography
\usepackage{xcolor}   
\usepackage[final]{corl_2021}
\usepackage{subcaption}
\newcommand{\E}[2]{\mathbb{E}_{#1}{\left[#2\right]}}

\newcommand{\mpc}{{\pi_{H,\hat{V}}}}
\newtheorem{lemma}{Lemma}
\newtheorem{assumption}{Assumption}
\newtheorem{corollary}{Corollary}
\newtheorem{theorem}{Theorem}
\newtheorem{theoremr}{Theorem}
\DeclareMathOperator*{\argmax}{argmax}
\algnewcommand{\algorithmicforeach}{\textbf{for}}
\algdef{SE}[FOR]{ForEach}{EndForEach}[1]
  {\algorithmicforeach\ #1\ \algorithmicdo}% \ForEach{#1}
  {\algorithmicend\ \algorithmicforeach}% \EndForEach

\title{Learning Off-Policy with Online Planning}

\author{
  Harshit Sikchi\thanks{Currently at The University of Texas at Austin, Email: hsikchi@utexas.edu},\, Wenxuan Zhou,\, David Held\\
  Robotics Institute\\
 Carnegie Mellon University \\
  \texttt{\{hsikchi, wenxuanz, dheld\}@cs.cmu.edu} \\
}

\begin{document}

\maketitle

\vspace{-8mm}
\begin{abstract} Reinforcement learning (RL) in low-data and risk-sensitive domains requires performant and flexible deployment policies that can readily incorporate constraints during deployment. One such class of policies are the semi-parametric H-step lookahead policies, which select actions using trajectory optimization over a dynamics model for a fixed horizon with a terminal value function. In this work, we investigate a novel instantiation of H-step lookahead with a~\textit{learned} model and a terminal value function learned by a~\textit{model-free off-policy} algorithm, named Learning Off-Policy with Online Planning (LOOP). We provide a theoretical analysis of this method, suggesting a tradeoff between model errors and value function errors and empirically demonstrate this tradeoff to be beneficial in deep reinforcement learning.
Furthermore, we identify the ``Actor Divergence’' issue in this framework and propose Actor Regularized Control (ARC), a modified trajectory optimization procedure. We evaluate our method on a set of robotic tasks for Offline and Online RL and demonstrate improved performance. We also show the flexibility of LOOP to incorporate safety constraints during deployment with a set of navigation environments. We demonstrate that LOOP is a desirable framework for robotics applications based on its strong performance in various important RL settings. Project video and details can be found at \href{https://hari-sikchi.github.io/loop}{hari-sikchi.github.io/loop}.


\end{abstract}

\section{Introduction}

\begin{wrapfigure}{r}{0.5\textwidth}
\vspace{-5mm}
\centering
\includegraphics[width=0.5\textwidth]{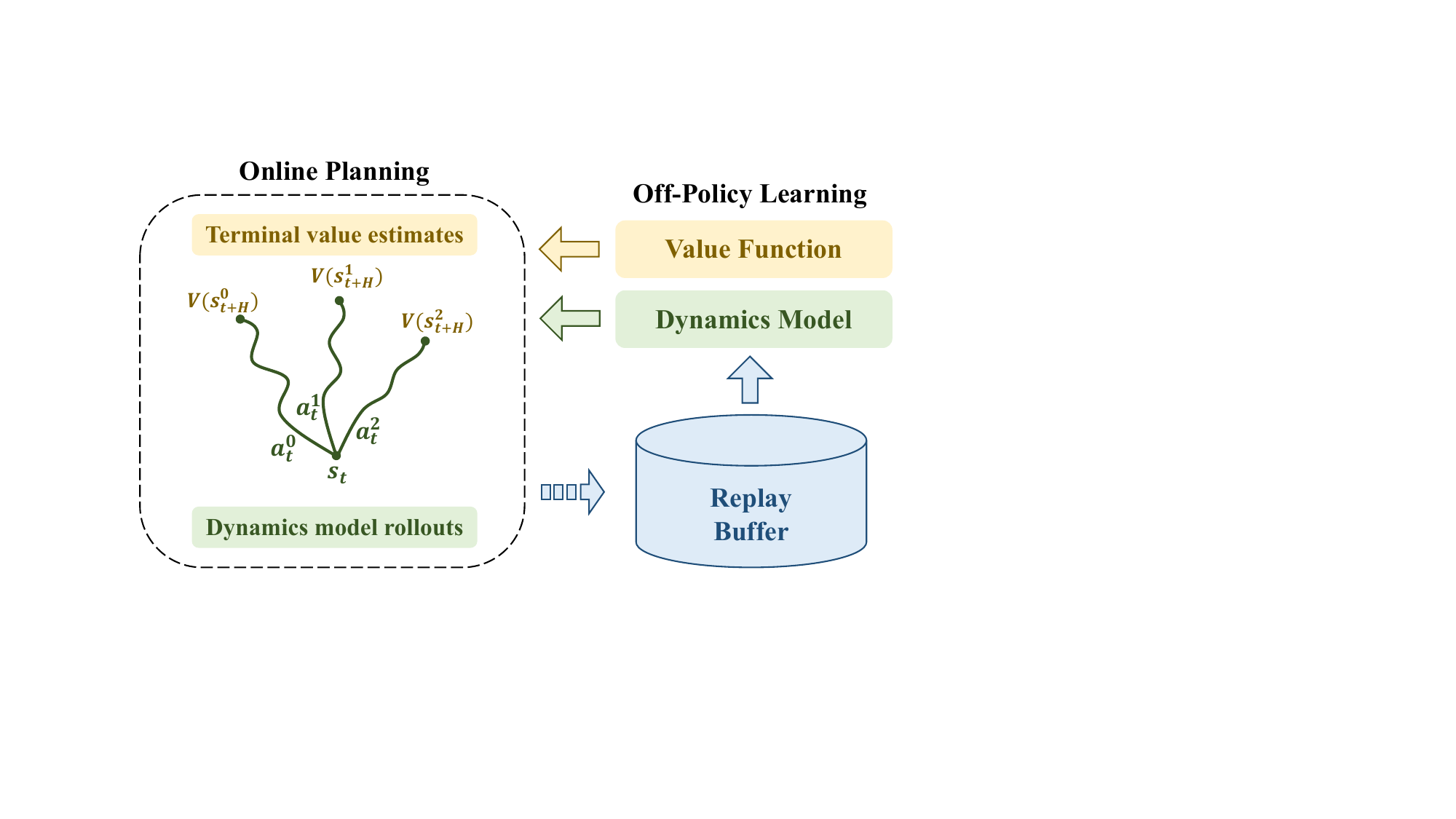}
\caption{Overview of LOOP: A learned dynamics model is utilized for Online Planning with a terminal value function. The value function is learned via a model-free off-policy algorithm.}
\label{fig:loop_overview}
\end{wrapfigure} 

Off-policy reinforcement learning algorithms have been widely used in many robotic applications due to their sample efficiency and their ability to incorporate data from different sources~\cite{kalashnikov2021mt,haarnoja2018learning,matas2018sim,kalashnikov2018qt}.
Model-free off-policy algorithms sample transitions from a replay buffer to learn a value function and then update the policy according to the value function~\cite{haarnoja2018soft, fujimoto2018addressing}. Thus, the performance of the policy is highly dependent on the estimation of the value function. However, learning an accurate value function from off-policy data is challenging especially in deep RL due to a variety of issues, such as overestimation bias~\cite{thrun1993issues,fujimoto2018off}, delusional bias~\cite{lu2018non}, rank loss~\cite{kumar2020implicit}, instability~\cite{fu2019diagnosing}, and divergence~\cite{achiam2019towards}. 
Another shortfall of model-free off-policy algorithms in continuous control is that the policy is usually parametrized by a feedforward neural network which lacks flexibility during deployment. 

Previous works in model-based RL have explored different ways of using a dynamics model to improve off-policy algorithms~\cite{janner2019trust,rajeswaran2020game,feinberg2018model,buckman2018sample, chua2018deep}. One way of incorporating the dynamics model is to use H-step lookahead policies~\cite{efroni2020online}. At each timestep, H-step lookahead policies rollout the dynamics model H-step into the future from the current state to find an action sequence with the highest return. Within this trajectory optimization process, a terminal value function is attached to the end of the rollouts to provide an estimation of the return beyond the fixed horizon. This way of online planning offers us a degree of explainability missing in fully parametric methods while also allowing us to take constraints into account during deployment.
Previous work proves faster convergence with H-step lookahead policies in tabular setting~\citep{efroni2020online} or showed improved sample complexity with a \textit{ground-truth} dynamics model~\citep{lowrey2018plan}. However, the benefit of H-step lookahead policies remains unclear under an \textit{approximate model} and an \textit{approximate value function}. Additionally, if H-step lookahead policies are used during the value function update~\cite{lowrey2018plan}, the required computation of value function update will be significantly increased.

In this work, we take this direction further by studying H-step lookahead both theoretically and empirically with three main contributions.  \textbf{First}, we provide a theoretical analysis of H-step lookahead under an \textit{approximate model} and \textit{approximate value function}.  Our analysis suggests a trade-off between model error and value function error, and we empirically show that this tradeoff can be used to improve policy performance in Deep RL. \textbf{Second}, we introduce Learning Off-Policy with Online Planning (LOOP) (Figure~\ref{fig:loop_overview}). To avoid the computational overhead of performing trajectory optimization while updating the value function as in previous work~\cite{lowrey2018plan}, the value function of LOOP is updated via a parameterized actor using a model-free off-policy algorithm (``Learning Off-Policy''). LOOP exploits the benefits of H-step lookahead policies when the agent is deployed in the environment during exploration and evaluation (``Online Planning''). This novel combination of model-based online planning and model-free off-policy learning provides sample-efficient and computationally-efficient learning. We also identify the ``Actor Divergence" issue in this combination and propose a modified trajectory optimization method called Actor Regularized Control (ARC). ARC performs implicit divergence regularization with the parameterized actor through Iterative Importance Sampling. 

\textbf{Third}, we explore the flexibility of H-step lookahead policies for improved performance  in offline RL and safe RL, which are both important settings in robotics. 
LOOP can be applied on top of various offline RL algorithms to improve their evaluation performance. LOOP's semiparameteric behavior policy also allows it to easily incorporate safety constraints during deployment. We evaluate LOOP on a set of simulated robotic tasks including locomotion, manipulation, and controlling an RC car. We show that LOOP provides significant improvement in performance for online RL, offline RL, and safe RL, which makes it a strong choice of RL algorithm for robotic applications.

\section{Related Work}

\textbf{Model-based RL} Model-based reinforcement learning (MBRL) methods learn a dynamics model and use it to optimize the policy. State-of-the-art model-based RL methods usually have better sample efficiency compared to model-free methods while maintaining competitive asymptotic performance~\cite{kurutach2018model,janner2019trust}. One approach in MBRL is to use trajectory optimization with a learned dynamics model~\cite{chua2018deep, nagabandi2019deep,deisenroth2011pilco}. These methods can reach optimal performance when a large enough planning horizon is used. However, they are limited by not being able to reason about the rewards beyond the planning horizon. Increasing the planning horizon increases the number of trajectories that need to be sampled and incurs a heavy computational cost.

Various attempts have been made to combine model-free and model-based RL. GPS~\cite{levine2013guided} combines trajectory optimization using analytical models with the on-policy policy gradient estimator. MBVE~\cite{feinberg2018model} and STEVE~\cite{buckman2018sample} use the model to improve target value estimates.  Approaches such as MBPO~\cite{janner2019trust} and MAAC~\cite{clavera2020model} follow Dyna-style~\cite{sutton1991dyna} learning where imagined short-horizon trajectories are used to provide additional transitions to the replay buffer leveraging model generalization. \citet{piche2018probabilistic} use Sequential Monte Carlo (SMC) to capture multimodal policies. The SMC policy relies on combining multiple 1-step lookahead value functions to sample a trajectory proportional to the unnormalized probability $\text{exp}(\sum_{i=1}^H(A(s,a)))$; this approach potentially compounds value function errors, in contrast to LOOP which uses single H-step lookahead planning for each state. POLO~\cite{lowrey2018plan} shows advantages of trajectory optimization under \textit{ground-truth} dynamics with a terminal value function. The value function updates involve additional trajectory optimization routines which is one of the issues we aim to address with LOOP. The computation of trajectory optimization in POLO is $\mathcal{O}(THN)$ while LOOP is $\mathcal{O}(TH)$ where $T$ is the number of environment timesteps, $H$ is the planning horizon, and $N$ is the number of samples needed for training the value function.

\textbf{Off-Policy RL} LOOP relies on a terminal value function for long horizon reasoning which can be learned effectively via model-free off-policy RL algorithms. Off-policy RL methods such as SAC~\cite{haarnoja2018soft} and TD3~\cite{fujimoto2018addressing} use the replay buffer to learn a Q-function that evaluates a parameterized actor and then optimize the actor by maximizing the Q-function. Off-policy methods can be modified to be used for Offline RL problems where the goal is to learn a policy from a static dataset~\cite{agarwal2020optimistic,fujimoto2018off,zhang2020gendice,siegel2020keep,levine2020offline,kumar2020conservative,zhou2020plas}. MBOP~\cite{argenson2020model}, a recent model-based offline RL method, leverages planning with a terminal value function, but the value function is a Monte Carlo evaluation of truncated replay buffer trajectories, whereas in LOOP the value function is trained for optimality under the dataset.
\section{Preliminaries}
A Markov Decision Process (MDP) is defined by the tuple $(\mathcal{S}, \mathcal{A}, p, r, \rho_0)$ with state-space $\mathcal{S}$, action-space $\mathcal{A}$, transition probability $p(s_{t+1}| s_t, a_t)$, reward function $r(s,a)$, and initial state distribution $\rho_0(s)$. In the infinite horizon discounted MDP, the goal of reinforcement learning algorithms is to maximize the return for policy $\pi$ given by  $J^\pi=\E{a_t\sim\pi(s_t),s_0 \sim \rho_0}{\sum_{t=0}^\infty\gamma^t r(s_t,a_t)}$.

\textbf{Value functions:} $V^\pi$ : $\mathcal{S}\rightarrow  \mathbb{R}$ represents a state-value function which estimates the return from the current state $s_t$ and following policy $\pi$, defined as $ V^\pi(s)=\E{a_t\sim\pi(s_t)}{\sum_{t=0}^\infty\gamma^t r(s_t,a_t)|s_0=s}$. Similarly, $Q^\pi$ : $\mathcal{S}\times\mathcal{A}\rightarrow  \mathbb{R}$ represents a action-value function, usually referred as a Q-function, defined as $Q^\pi(s,a)=\E{a_t\sim\pi(s_t)}{\sum_{t=0}^\infty\gamma^t r(s_t,a_t)|s_0=s,a_0=a}$.
Value functions corresponding to the optimal policy $\pi^*$ are defined to be $V^*$ and $Q^*$. The value function can be updated according to the Bellman operator $\mathcal{T}$:
\begin{equation}
        \mathcal{T}{Q}(s_t,a_t) = r(s_t,a_t)+\E{s_{t+1}\sim p,a_{t+1}\sim\pi_Q}{\gamma( Q(s_{t+1},a_{t+1})}
\end{equation}
where $\pi_{Q}$ is updated to be greedy with respect to $Q$, the current Q-function.

\textbf{Constrained MDP for safety:} A constrained MDP (CMDP) is defined by the tuple $(\mathcal{S}, \mathcal{A}, p, r,c, \rho_0)$ with an additional cost function $c(s,a)$. We define the cumulative cost of a policy to be $D^{\pi}=\E{a_t\sim\pi(s_t),s_0\sim\rho_0}{\sum_{t=0}^\infty \gamma^t c(s_t,a_t)}$. A common objective for safe reinforcement learning is to find a policy $\pi= \text{argmax}_{\pi}{J^{\pi}}$ subject to $D^\pi\le d_0$ where $d_0$ is a safety threshold~\cite{altman1999constrained}.

\section{H-step Lookahead with Learned Model and Value Function}

Model-based algorithms often learn an approximate dynamics model $\hat{M}(s_{t+1}|s_t,a_t)$ using the data collected from the environment. One way of using the model is to find an action sequence that maximizes the cumulative reward with the learned model using trajectory optimization~\cite{williams2016aggressive,rubinstein1999cross, camacho2013model}. An important limitation of this approach is that the computation grows exponentially with the planning horizon. Thus, methods like~\cite{williams2016aggressive,chua2018deep,nagabandi2019deep, wang2019exploring, zhang2021importance} plan over a fixed, short horizon and are unable to reason about long-term reward. Let $\pi_{H}$ be such a fixed horizon policy: 
\begin{align}
    \label{eq:fixed-horizon-policy}
     \pi_{H}(s_0)=&\argmax_{a_0}\max_{a_1,..,a_{H-1}}\E{ 
     \hat{M}}{R_{H}(s_0,\tau)}~,\text{where } R_{H}(s_0,\tau)=\sum_{t=0}^{H-1} \gamma^t r(s_t,a_t)   
\end{align}
where $\tau$ denotes the action sequence $a_{[0..{H-1}]}$. One way to enable efficient long-horizon reasoning is to augment the planning trajectory with a terminal value function. Given a value-function $\hat{V}$, we define a policy $\pi_{H,\hat{V}}$ obtained by maximizing the H-step lookahead objective:
\begin{align}
\label{eq:H-step_objective}
    \pi_{H,\hat{V}}(s_0)=&\argmax_{a_0}\max_{a_1,..,a_{H-1}}\E{  \hat{M}}{ R_{H,\hat{V}}(s_0,\tau)}\\
    &\text{where } R_{H,\hat{V}}(s_0,\tau)=\sum_{t=0}^{H-1} \gamma^t r(s_t,a_t)+\gamma^{H}\hat{V}(s_{H}) \nonumber
\end{align}
The quality of both the model $\hat{M}$ and the value-function $\hat{V}$ affects the performance of the overall policy.
%$\pi_{H,\hat{V}}$. 
To show the benefits of this combination of model-based trajectory optimization and the value-function, we now analyze and bound the performance of the H-step look-ahead policy $\pi_{H,\hat{V}}$ compared to its fixed-horizon counterpart without the value-function $\pi_{H}$ (Eqn.~\ref{eq:fixed-horizon-policy}), as well as the greedy policy obtained from the value-function $\pi_{\hat{V}}=\text{argmax}_{a}\E{s'\sim M(.|s,a)}{r(s,a)+\gamma \hat{V}(s')}$. Following previous work, we will construct the proofs with the state-value function $V$, but the proofs for the action-value function $Q$ can be derived similarly.

\begin{lemma}
\label{lemma:greedy_lookahead}
~(\citet{singh1994upper}) Suppose we have an approximate value function $\hat{V}$ such that $\max_s|V^*(s)-\hat{V}(s)|\le\epsilon_v$. Then the performance of the 1-step greedy policy $\pi_{\hat{V}}$ can be bounded as:\\
\vspace{-2mm}
\begin{equation}
\label{eq:greedy_lookahead}
     J^{\pi^*}-J^{\pi_{\hat{V}}}\le \frac{\gamma}{1-\gamma}[2\epsilon_v]
\end{equation}
\end{lemma}
\vspace{-2mm}
\begin{theorem}
\label{thm:h_step_thm}
(H-step lookahead policy) Suppose $\hat{M}$ is an approximate dynamics model with Total Variation distance bounded by $\epsilon_m$. Let $\hat{V}$ be an approximate value function such that $\max_s|V^*(s)-\hat{V}(s)|\le\epsilon_v$. Let the reward function $r(s,a)$ be bounded by [0,$R_{\text{max}}$] and $\hat{V}$ be bounded by [0,$V_{\text{max}}$]. \reb{Let $\epsilon_p$ be the suboptimality incurred in H-step lookahead optimization (Eqn.~\ref{eq:H-step_objective})}. Then the performance of the H-step lookahead policy $\pi_{H,\hat{V}}$ can be bounded as:\\
\vspace{-2mm}
\begin{equation}
\label{eq:lookahead_bound}
     J^{\pi^*}-J^{\pi_{H,\hat{V}}}\le \frac{2}{1-\gamma^H}[C(\epsilon_m,H,\gamma)\reb{+\frac{\epsilon_p}{2}}+\gamma^H\epsilon_v]
\end{equation}
\vspace{-2mm}
where 
\vspace{-2mm}
\begin{equation*}
    C(\epsilon_m,H,\gamma)=R_{\max} \sum_{t=0}^{H-1}\gamma^t t \epsilon_m + \gamma^H H\epsilon_mV_{\text{max}}
\end{equation*}
\vspace{-5mm}

\end{theorem}
\vspace{-0.2cm}
\begin{proof}
Due to the page limit, we defer the proof to Appendix~\ref{ap:A1Hstep}. We also provide extension of Theorem~\ref{thm:h_step_thm} under assumptions on model generalization and concentrability in Corollary \ref{thm:h_step_model_gen} and Theorem~\ref{thm:h_step_fitted_q} respectively in Appendix~\ref{ap:theory}.
\end{proof}
\vspace{-0.35cm}
%\subsubsection{$\pi_{H,\hat{Q}}$ vs. $\pi_{H}$}
% \subsubsection{H-step Look-ahead Policy vs H-step Fixed Horizon Policy}
\textbf{H-step Lookahead Policy vs H-step Fixed Horizon Policy:}
% \hs{Why take Q to be const? For h-step policies we always assume terminal Q to be 0? It ll make the statements more clear}
The fixed-horizon policy $\pi_{H}$ can be considered as a special case of $\pi_{H, \hat{V}}$ with $\hat{V}(s)=0 ~\forall s \in \mathcal{S}$. Following Theorem \ref{thm:h_step_thm}, $\epsilon_{\hat{V}}=\max_s|V^*(s)|$ implies a potentially large optimality gap. This suggests that learning a value function that better approximates $V^*$ than $\hat{V}(s)=0$ will give us a smaller optimality gap in the worst case.

\textbf{H-step lookahead policy vs 1-step greedy policy:} By comparing Lemma~\ref{lemma:greedy_lookahead} and Theorem~\ref{thm:h_step_thm}, we observe that the performance of the H-step lookahead policy $\pi_{H,\hat{V}}$ reduces the dependency on the value function error $\epsilon_v$ at least by a factor of $\gamma^{H-1}$ while introducing an additional dependency on the model error $\epsilon_m$. This implies that the H-step lookahead is beneficial when the value-function bias dominates the bias in the learned model. In the low data regime, the value function bias can result from compounded sampling errors~\cite{agarwal2019reinforcement} and is likely to dominate the model bias, as evidenced by the success of model-based RL methods in the low-data regime~\cite{argenson2020model,matsushima2020deployment,janner2019trust}; we observe this hypothesis to be consistent with our experiments where H-step lookahead offers large gains in sample efficiency. Further, errors in value learning with function approximation can stem from a number of reasons explored in previous work, some of them being Overestimation, Rank Loss, Divergence, Delusional bias, and Instability~\cite{thrun1993issues,fu2019diagnosing,fujimoto2018addressing,kumar2019stabilizing,kumar2020implicit}.  Although this result may be intuitive to many practitioners, it has not been shown theoretically; further, we demonstrate that we can use this insight to improve the performance of state-of-the-art methods for online RL, offline RL, and safe RL.
\section{Learning Off-Policy with Online Planning}
We propose Learning Off-Policy with Online Planning (LOOP) as a framework of using H-step lookahead policies that combines online trajectory optimization with model-free off-policy RL (Figure~\ref{fig:loop_overview}).  We use the replay buffer to learn a dynamics model and a value function using an off-policy algorithm. The H-step lookahead policy (Eqn.~\ref{eq:H-step_objective}) generates rollouts using the dynamics model with a terminal value function and selects the best action for execution. The underlying off-policy algorithm is boosted by the H-step lookahead which improves the performance of the policy during both exploration and evaluation. 
From another perspective, the underlying model-based trajectory optimization is improved using a terminal value function for reasoning about future returns. 
In this section, we discuss the Actor Divergence issue in the LOOP framework and introduce additional applications and instantiations of LOOP for offline RL and safe RL.

\subsection{Reducing actor-divergence with Actor Regularized Control (ARC)}
\label{sec:actor_divergence}
As discussed above, LOOP utilizes model-free off-policy algorithms to learn a value function 
%close to $V^*$ 
in a more computationally efficient manner.  It relies on actor-critic methods which use a parametrized actor $\pi_{\theta}$ to facilitate the Bellman backup. However, we observe that combining trajectory optimization and policy learning naively will lead to an issue that we refer to as ``actor divergence": a different policy is used for data collection (H-step lookahead policy $\pi_{H, \hat{V}}$) than the policy that is used to learn the value-function (the parametrized actor $\pi_{\theta}$). 
This leads to a potential distribution shift between the state-action visitation distribution between the parametrized actor $\pi_{\theta}$ and the actual behavior policy $\pi_{H, \hat{V}}$ which can lead to accumulated bootstrapping errors with the Bellman update and destabilize value learning~\cite{kumar2019stabilizing}. One possible solution in this case is to use Offline RL~\cite{levine2020offline}; however, in practice, we observe that offline RL in this setup leads to learning instabilities.  We defer discussion on this alternative to the Appendix~\ref{ap:offline_rl_in_loop}.
Instead, we propose to resolve the actor-divergence issue via a modified trajectory optimization method called Actor Regularized Control (ARC).

In ARC, we aim to constrain the action selection of the trajectory optimization to be close to the parametrized actor. We frame the following general constrained optimization problem for policy improvement~\cite{vieillard2020leverage}:
\begin{equation}
\label{eq:ARC}
    p^\tau_{opt}=\argmax_{p^\tau} \mathbb{E}_{p^\tau} [L_{H,\hat{V}}(s_t,\tau)]~,~\textrm{s.t}~~D_{KL}(p^\tau||p^\tau_{prior})\le \epsilon
\end{equation}
where $L_{H,\hat{V}}(s_t,\tau)$ is the expected lookahead objective (Eqn.~\ref{eq:H-step_objective}) under the learned model given by $L_{H,\hat{V}}(s_t,\tau)=\E{\hat{M}}{R_{H,\hat{V}}(s_t,\tau)}$, starting from state $s_t$, $p^\tau$ is a distribution over action sequences $\tau$ of horizon H starting from $s_t$, and $p^\tau_{prior}$ is a prior distribution over such action sequences. We will use the parametrized actor to derive this prior in ARC. This optimization admits a closed form solution by enforcing the KKT conditions where the optimal policy is given by $p^\tau_{opt} \propto p^\tau_{prior} e^{\frac{1}{\eta}L_{H,\hat{V}}(s_t,\tau)}$~\cite{nair2020accelerating, peters2007reinforcement, peters2008natural,peters2010relative}, where $\eta$ is the lagrangian dual variable. The above formulation generalizes a number of prior work~\cite{haarnoja2018soft,williams2016aggressive, nair2020accelerating} (more details in Appendix~\ref{ap:arc_algo}).

Approximating the optimal policy $p^\tau_{opt}$ as a multivariate gaussian with diagonal covariance $\hat{p}^\tau_{opt}=\mathcal{N}(\mu_{opt},\sigma_{opt})$ , the parameters can be estimated using importance sampling under the proposal distribution $p^\tau_{prior}$ as:
\begin{equation}
    \hat{p}^\tau_{opt}=\mathcal{N}(\mu_{opt},\sigma_{opt})~,~\mu_{opt} = \E{\tau',\hat{M}}{\frac{p^\tau_{opt}(\tau')}{p^\tau_{prior}(\tau')}\tau'}~,~ \sigma_{opt} = \E{\tau',\hat{M}}{\frac{p^\tau_{opt}(\tau')}{p^\tau_{prior}(\tau')}(\tau'-\mu)^2}
\end{equation}
where $\tau'\sim p^\tau_{prior}$. We use iterative importance sampling to estimate $\hat{p}^\tau_{opt}$ which is parameterized as a Gaussian whose mean and variance  at iteration $m+1$ are given by the empirical estimate:
\begin{align}
    \mu^{m+1} = \frac{\sum_{i=1}^N[e^{\frac{1}{\eta} L_{H,\hat{V}}(s_t,\tau')}\tau']}{\sum_{i=1}^N e^{\frac{1}{\eta} L_{H,\hat{V}}(s_t,\tau')}}~,~\sigma^{m+1} = \frac{\sum_{i=1}^N[e^{\frac{1}{\eta} L_{H,\hat{V}}(s_t,\tau')}(\tau'-\mu^{m+1})^2]}{\sum_{i=1}^N e^{\frac{1}{\eta} L_{H,\hat{V}}(s_t,\tau')}}
\end{align}
where $\tau'\sim \mathcal{N}(\mu^m,\sigma^m)$ and $\mathcal{N}(\mu^0,\sigma^0)$ is set to $p^\tau_{prior}$. As long as we perform a finite number of iterations, the final trajectory distribution is constrained in total variation to be close to the prior as a result of finite trust region updates as shown in Lemma~\ref{thm:implicit_kl} in Appendix~\ref{ap:lemma2_proof}.

To reduce actor divergence in LOOP, we constrain the action-distribution of the trajectory optimization to be close to that of the parametrized actor $\pi_\theta$. To do so, we set $p^\tau_{prior}=\beta \pi_\theta+(1-\beta)\mathcal{N}(\mu_{t-1},\sigma)$. The trajectory prior is a mixture of the parametrized actor and the action sequence from the previous environment timestep with additional Gaussian noise $\mathcal{N}(0,\sigma)$. Using 1-timestep shifted solution from the previous timestep allows to amortize trajectory optimization over time~\cite{argenson2020model}.  For online RL, we can vary $\sigma$ to vary the amount of exploration during training. For offline RL, we set $\beta = 1$ to constrain actions to  be  close  to  those in the dataset  (from which $\pi_\theta$ is learned) to  be  more  conservative.

%%%%%%%%%%%%%%%%%%%%%%%%%%%%%%%%%%%%%%%%%%%%%%%%%%%%%%%%%%%%%%%%%%%%%%%%%%%%%%%%%%%%%%
\subsection{Additional instantiations of LOOP: Offline-LOOP and Safe-LOOP}
LOOP not only improves the performance of previous model-based and model-free RL algorithms but also shows versatility in different settings such as the offline RL setting and the safe RL setting. These potentials of H-step lookahead policies have not been explored in previous work. 

\textbf{LOOP for Offline RL:}
In offline reinforcement learning, the policy is learned from a static dataset without further data collection. We can use LOOP on top of an existing off-policy algorithm as a plug-in component to improve its test time performance by using the model-based rollouts as suggested by Theorem~\ref{thm:h_step_thm}. Note that this is different from the online setting in the previous section in which LOOP also influences exploration.
%acts a robust policy improvement while influencing exploration. 
In offline-LOOP, to account for the uncertainty in the model and the Q-function, ARC optimizes for the following uncertainty-pessimistic objective similar to~\cite{yu2020mopo,kidambi2020morel}:
\begin{equation}
\label{eq:offline_loop}
  \text{mean}_{[K]}[R_{H,\hat{V}}(s_t,\tau)] - \beta_{pess}  \text{std}_{[K]}[R_{H,\hat{V}}(s_t,\tau)]    
\end{equation}
 where $[K]$ are the model ensembles, $\beta_{pess}$ is the pessimism parameter and $R_{H,\hat{V}}$ is the H-horizon lookahead objective defined in Eqn.~\ref{eq:H-step_objective}. % Note that the ``std'' is the standard deviation of the mean outputs of the models. 

\label{safety}
\textbf{Safe Reinforcement Learning:}
Another benefit of LOOP with its semi-parameteric policy is that we can easily incorporate (possibly non-stationary) constraints
%during training as well as testing 
with the model-based rollout, while being an order of magnitude more sample efficient than existing safe model-free RL algorithms. To account for safety in the planning horizon, ARC optimizes for the following cost-pessimistic objective:
\vspace{-2.5mm}
\begin{equation}
\label{eq:safe_loop}
    \text{argmax}_{a_t} \E{\hat{M}}{R_{H,\hat{V}}(s_t,\tau)}
    \text{s.t. } \max_{[K]} \sum_{t=t}^{t+H-1}\gamma^tc(s_t,a_t)\le d_0
\end{equation}\vspace{-1mm}
where $[K]$ are the model ensembles, $c$ is the constraint cost function and $R_{H,\hat{V}}$ is the H-horizon lookahead objective defined in Eqn.~\ref{eq:H-step_objective} and $d_0$ is the constraint threshold. For each action rollout, the worst-case cost is considered w.r.t model uncertainty to be more conservative. The pseudocode for modified ARC to solve the above constrained optimization is given in Appendix~\ref{ap:safeARC_algo}.

\begin{figure*}
% \vspace{2mm}
    \centering
    \begin{subfigure}[t]{0.19\linewidth}
        \centering
        \includegraphics[width=\linewidth]{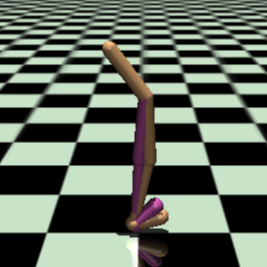}
    \end{subfigure}%
    ~ 
    \begin{subfigure}[t]{0.19\linewidth}
        \centering
        \includegraphics[width=\linewidth]{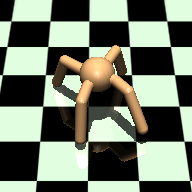}
    \end{subfigure}%
    ~ 
    \begin{subfigure}[t]{0.19\linewidth}
        \centering
        \includegraphics[width=\linewidth]{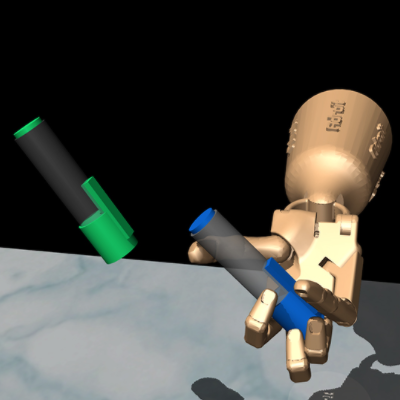}
    \end{subfigure}%
    ~ 
    \begin{subfigure}[t]{0.19\linewidth}
        \centering
        \includegraphics[width=\linewidth]{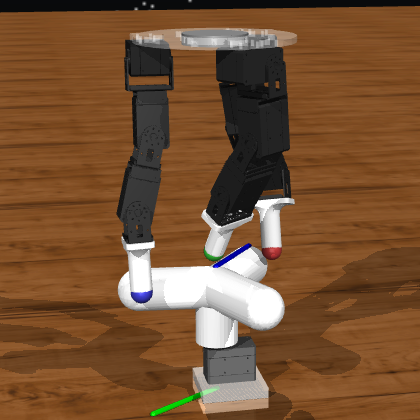}
    \end{subfigure}%
    ~ 
    \begin{subfigure}[t]{0.19\linewidth}
        \centering
        \includegraphics[width=\linewidth]{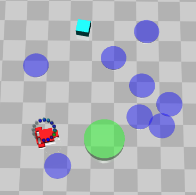}
    \end{subfigure}
    \caption{ We evaluate LOOP over a variety of environments ranging from locomotion, manipulation to navigation including Walker2d-v2, Ant-v2, PenGoal-v1,   Claw-v1,  CarGoal1, etc.}
    \vspace{-3mm}
    \label{fig:envs}
\end{figure*}
\vspace{-0.4cm}
\section{Experimental Results}
\label{sec:result}
\vspace{-2mm}

In the experiments, we evaluate the performance of LOOP combined with different off-policy algorithms in the settings of online RL, offline RL and safe RL over a variety of environments (Figure~\ref{fig:envs}). Implementation details of LOOP and the baselines can be found in Appendix~\ref{ap:exp_details}.

\subsection{LOOP for Online RL}
In this section, we evaluate the performance of LOOP for online RL on three OpenAI Gym MuJoCo~\cite{todorov2012mujoco} locomotion control tasks: {\fontfamily{qcr}\selectfont
HalfCheetah-v2, Walker-v2, Ant-v2}  and two manipulation tasks: {\fontfamily{qcr}\selectfont
PenGoal-v1, Claw-v1}. In these experiments, we use Soft Actor-Critic (SAC)~\cite{haarnoja2018soft} as the underlying off-policy method with the ARC optimizer described in Section~\ref{sec:actor_divergence}. 
%We observe that ARC-R is too restrictive for the online RL case. 
Further experiments on {\fontfamily{qcr}\selectfont InvertedPendulum-v2, Swimmer, Hopper-v2} and {\fontfamily{qcr}\selectfont Humanoid-v2} and more details on the baselines can be found in Appendix~\ref{ap:online_rl_exp} and Appendix~\ref{ap:online_rl} respectively. 

\vspace{3mm}
\begin{figure}[h]
\begin{center}
      \includegraphics[width=1.0\linewidth]{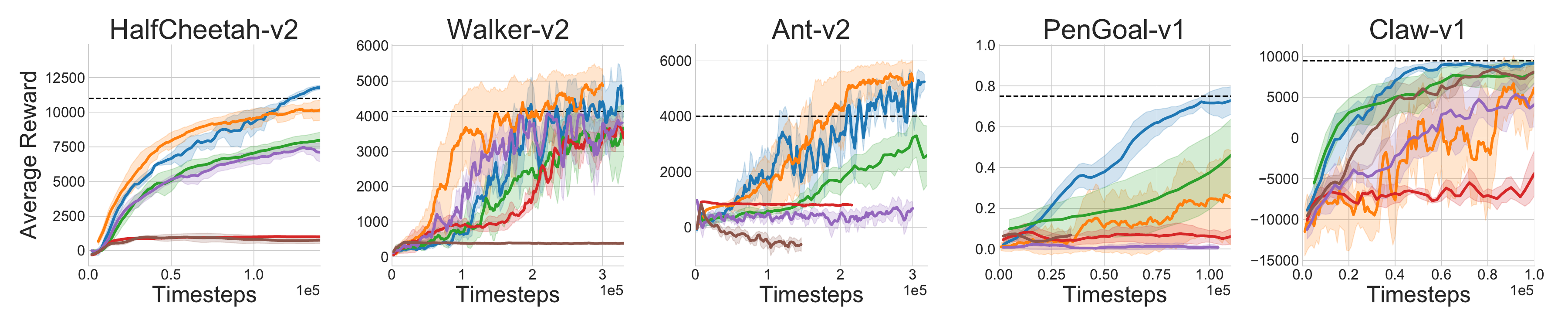}
      \\ 
    \includegraphics[width=1.0\linewidth]{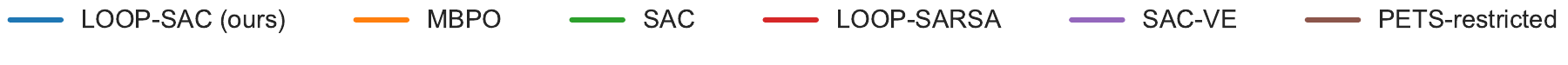}  
    \vspace{-7.2mm}
\end{center}
\caption{Comparisons of LOOP and the baselines for online RL. LOOP-SAC is significantly more sample efficient than SAC. It is competitive to MBPO for locomotion tasks and outperforms MBPO for manipulation tasks (PenGoal-v1 and Claw-v1). The dashed line indicates the performance of SAC at 1 million timesteps. Additional results on more environments can be found in Appendix~\ref{ap:online_rl_exp}. }
\label{fig:loop_main}
\end{figure}
\vspace{2mm}

\textbf{Baselines:} 
We compare the LOOP framework against the following baselines:
PETS-restricted, a variant of PETS~\cite{chua2018deep} that uses trajectory optimization (CEM) for the same horizon as LOOP but without a terminal value function. LOOP-SARSA uses a terminal value function which is an evaluation of the replay buffer policy, similar to MBOP~\cite{argenson2020model} in spirit. To compare with other ways of combining model-based and model-free RL, we also compare against MBPO~\cite{janner2019trust} and SAC-VE. MBPO leverages the learned model to generate additional transitions for value function learning. SAC-VE utilizes the model for value expansion, similar to MBVE~\cite{feinberg2018model} but uses SAC as the model-free component for a fair comparison with LOOP as done in~\cite{janner2019trust}. We do not include comparison to STEVE~\cite{buckman2018sample} or SLBO~\cite{luo2018algorithmic} as they were shown to be outperformed by MBPO, and perform poorly compared to SAC in Hopper and Walker environments ~\cite{janner2019trust}. We were unable to reproduce the results for SMC~\cite{piche2018probabilistic} due to missing implementation. We did not include POLO here due several reasons. An extended discussion can be found in Appendix~\ref{ap:polo_experiments}.

\begin{wrapfigure}{r}{0.5\textwidth}
    \includegraphics[width=1.0\linewidth]{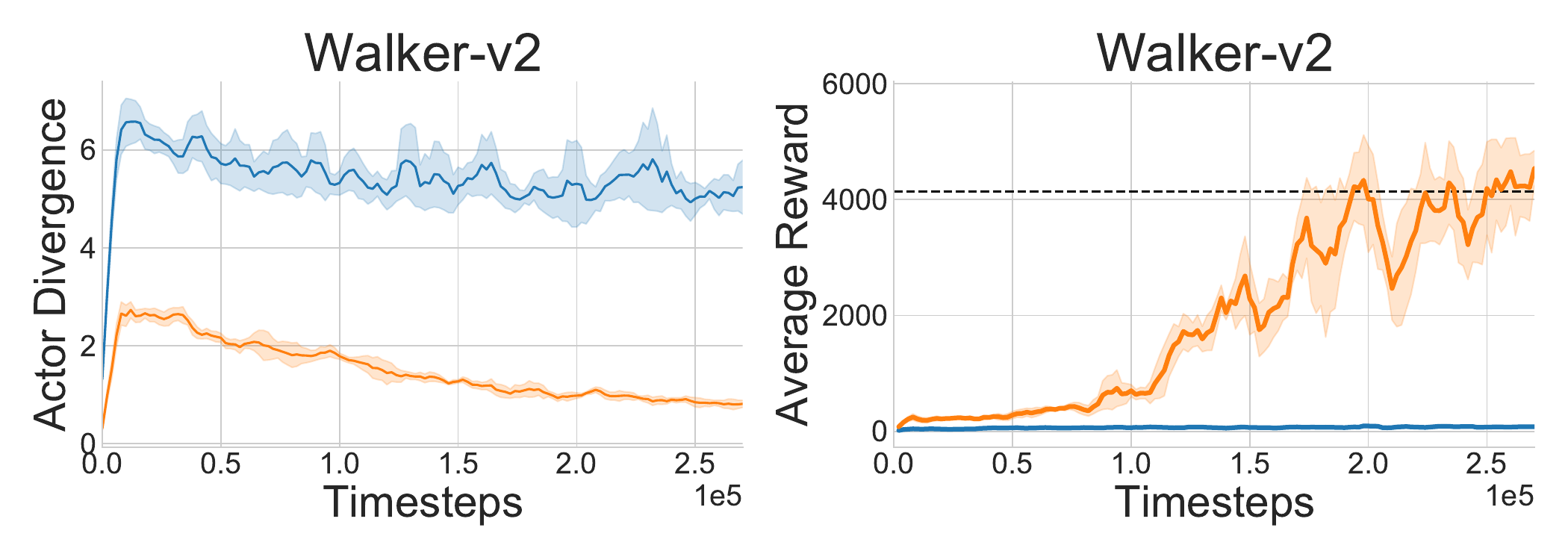}
    \includegraphics[width=1.0\linewidth]{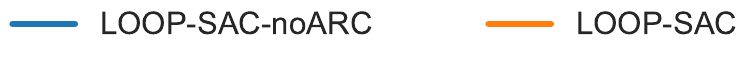}
\vspace{-5mm}
 \caption{(Left) ARC reduces the actor-divergence measured by the L2 distance between the mean of the parametrized actor and the output of the H-step lookahead policy. (Right) In absence of ARC, policy learning can be unstable.}
    \label{fig:ag_ablation_main}
\end{wrapfigure}
\textbf{Performance:} From Figure~\ref{fig:loop_main}, we observe that LOOP-SAC is significantly more sample efficient  than SAC, the underlying model-free method used to learn a terminal value function. LOOP-SAC also scales well to high-dimensional environments like {\fontfamily{qcr}\selectfont
Ant-v2} and {\fontfamily{qcr}\selectfont
PenGoal-v1}. PETS-restricted performs poorly due to myopic reasoning over a limited horizon $H$. SAC-VE and MBPO represent different ways of incorporating a model to improve off-policy learning. LOOP-SAC outperforms SAC-VE and performs competitively to MBPO, outperforming it significantly in {\fontfamily{qcr}\selectfont PenGoal-v1} and {\fontfamily{qcr}\selectfont Claw-v1}. In principle, methods like MBPO and value expansion can be combined with LOOP to potentially increase performance; we leave such combinations for future work. LOOP-SARSA has poor performance as a result of the poor value function that is trained for evaluating replay buffer policy rather than optimality. As an ablation study, we also run experiments using LOOP without ARC, which optimizes the unconstrained objective of Eqn.~\ref{eq:H-step_objective} using CEM~\cite{rubinstein1999cross}. Figure~\ref{fig:ag_ablation_main} (left) shows that ARC reduces actor-divergence effectively and Figure~\ref{fig:ag_ablation_main} (right) shows that learning performance is poor in absence of ARC for {\fontfamily{qcr}\selectfont Walker-v2}. More ablation results can be found in Appendix~\ref{ap:arc_experiments}.
\vspace{-0.4cm}
\subsection{LOOP for Offline RL}
\begin{table*}[h]
    % \vspace{2mm}
    \centering
    \footnotesize
    % \resizebox{\textwidth}{!}{%
    \begin{tabular}{c|c|c|c|c|c|c|c|c}
    \toprule
    \multirow{2}{*}{Dataset} & Env & CRR & LOOP & Improve\% & PLAS & LOOP & Improve\%& MBOP \\
    &&&CRR&&&PLAS&\\
    % \midrule
    % \multirow{3}{*}{random}& hopper&10.40&10.68&2.7&10.35&10.71&3.5&\textbf{10.8}\\
    % & halfcheetah&4.23&7.55&78.5&26.05&\textbf{26.14}&0.3&6.3\\
    % & walker2d&1.94&2.04&5.2&0.89&2.83&218.0&\textbf{8.1}\\
    \midrule
    \multirow{3}{*}{medium}& hopper&65.73&\textbf{85.83}&30.6&32.08&56.47&76.0&48.8\\
    & halfcheetah&41.14&41.54&1.0&39.33&39.54&0.5&\textbf{44.6}\\
    & walker2d&69.98&\textbf{79.18}&13.1&46.20&52.66&14.0& 41.0\\
    \midrule
    \multirow{3}{*}{med-replay}& hopper&27.69&29.08&5.0&29.29&\textbf{31.29}&6.8& 12.4\\
    & halfcheetah&42.29&42.84&1.3&43.96&\textbf{44.25}&0.7&42.3\\
    & walker2d&19.84&27.30&37.6&35.59&\textbf{41.16}&15.7&9.7\\
    % \midrule
    % \multirow{3}{*}{med-expert}& hopper&112.02&113.71&1.5&110.95&\textbf{114.32}&3.0&55.1\\
    % & halfcheetah&21.48&24.19&12.6&93.08&98.16&5.5&\textbf{105.9}\\
    % & walker2d&103.77&\textbf{105.76}&1.9&90.07&99.03&9.9&70.2\\
    \bottomrule
    \end{tabular}
\caption{Normalized scores for LOOP on the D4RL datasets comparing to the underlying offline RL algorithms and a baseline MBOP. LOOP improves the base algorithm across various types of datasets and environments. } 
\label{table:offline_rl_results}
\end{table*}

For Offline RL, we benchmark the performance using the D4RL datasets~\cite{fu2020d4rl}. We combine LOOP with two value-based offline RL algorithms: Critic Regularized Regression (CRR)~\cite{wang2020critic} and Policy in Latent Action Space (PLAS)~\cite{zhou2020plas}. We use the original offline RL algorithms to train a value function from the static data and then use it as the terminal value function for LOOP. We use $\beta = 1$ in the trajectory prior of ARC (Section~\ref{sec:actor_divergence}) in the offline RL setting to keep the policy conservative.

\textbf{Baselines:}
In addition to the underlying offline RL algorithms, we also include recent work MBOP~\cite{argenson2020model} as a baseline. MBOP uses a terminal value function which is an evaluation of the dataset policy. In contrast, LOOP uses a terminal value function trained with offline RL algorithms which is more optimal.

\textbf{Performance:} Table~\ref{table:offline_rl_results} presents the comparison of LOOP and the underlying offline RL algorithms. LOOP offers an average improvement of 15.91$\%$ over CRR and 29.49$\%$ over PLAS on the complete D4RL MuJoCo Locomotion dataset. Full results can be found in Appendix~\ref{ap:offline_rl_exp}. The results further highlight the benefit of the LOOP framework compared to the underlying model-free algorithms. 

\begin{figure*}[h]
\begin{multicols}{2}
    \includegraphics[width=1.0\linewidth]{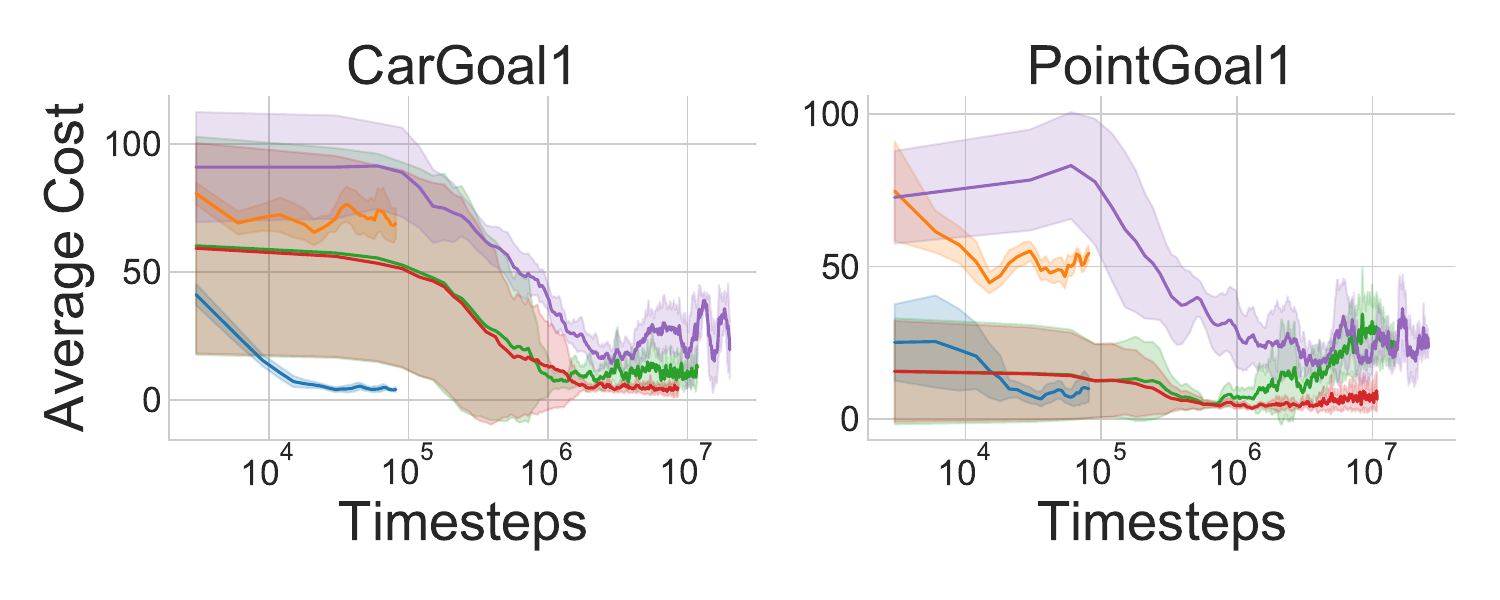}\par 
    \includegraphics[width=1.0\linewidth]{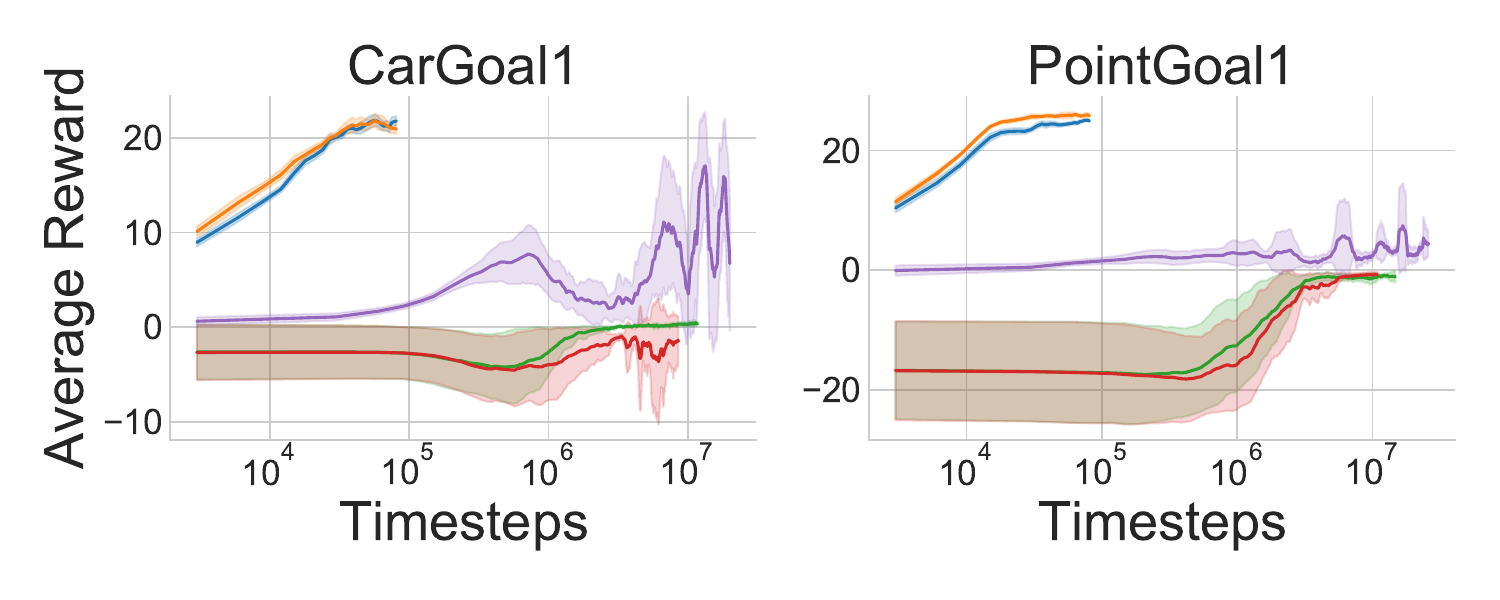} 
    \end{multicols}
    \vspace{-0.6cm}
    \begin{center}
        \includegraphics[width=1\linewidth]{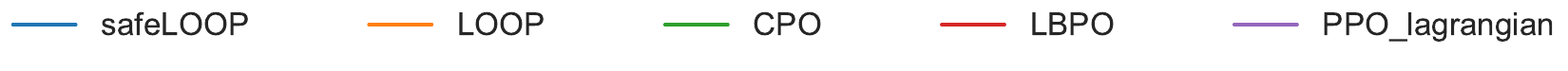}
    \end{center}
    \vspace{-5.0mm}
\caption{We compare safeLOOP with other safety methods such as CPO, LBPO, and PPO-lagrangian on OpenAI Safety Gym environments. It shows significant sample efficiency while offering similar or better safety benefits as the baselines.}
\label{fig:openai_safety}
\vspace{-5mm}
\end{figure*}
% \vspace{-2mm}
\subsection{LOOP for Safe RL}
\vspace{-2mm}
For safe RL, we modify the H-step lookahead optimization to maximize the sum of rewards while satisfying the cost constraints, as described in Section~\ref{safety}. We evaluate our method on two environments from the OpenAI Safety Gym~\citep{Ray2019} and an RC-car simulation environment~\cite{ahn2019towards}. The objective of the Safety Gym environments is to move a Point mass agent or a Car agent to the goal while avoiding obstacles. The RC-car environment is rewarded for driving along a circle of 1m fixed radius with a desired velocity while staying within the 1.2m circle during training. Details for the environments can be found in Appendix~\ref{ap:safe_rl}.

\begin{wrapfigure}{r}{0.5\textwidth}
\vspace{-7mm}
\begin{multicols}{2}
    \includegraphics[width=1.05\linewidth]{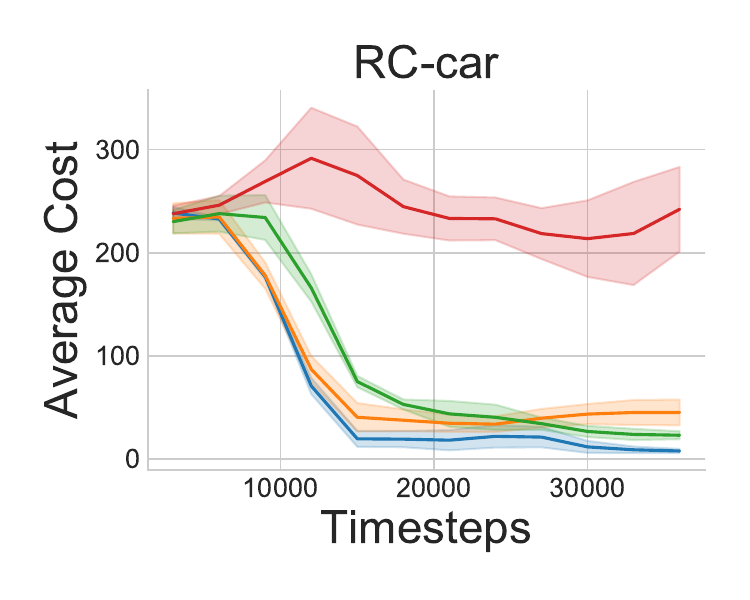}\par
    \includegraphics[width=1.05\linewidth]{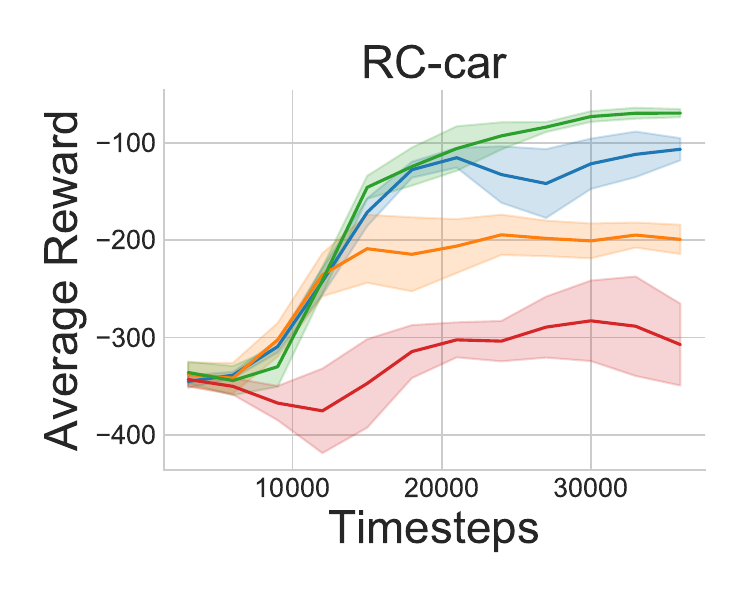}\par 
    \end{multicols}
    \vspace{-0.5cm}
     \includegraphics[width=1.0\linewidth]{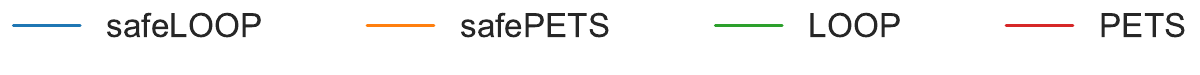}   
\caption{RC-car experiments show the importance of the terminal value function in the LOOP framework. SafeLOOP achieves higher returns than safePETS while being competitive in safety performance. Both safePETS and PETS fail to learn a drifting policy due to limited lookahead.}
\vspace{2mm}
\label{fig:safety_rc_car}
\end{wrapfigure}

\textbf{Baselines}: We compare our safety-augmented LOOP (safeLOOP) against various state-of-the-art safe learning methods such as CPO~\cite{achiam2017constrained}, LBPO~\cite{sikchi2021lyapunov}, and PPO-lagrangian~\cite{altman1996constrained,Ray2019}. CPO uses a trust region update rule that guarantees safety. LBPO relies on a barrier function formulated around a Lyapunov constraint for safety. PPO-lagrangian uses dual gradient descent to solve the constrained optimization.  
To ensure a fair comparison, all policies and dynamics models are randomly initialized, as is commonly done in safe RL experiments (rather than starting from a safe initial policy). We additionally compare against a model-based safety method that modifies PETS for safe exploration (safePETS) without the terminal value function.  We mostly compare to model-free baselines due to a lack of safe model-based Deep-RL baselines in the literature.

\textbf{Performance:}
For the OpenAI Safety Gym environments, we observe in Figure~\ref{fig:openai_safety} that safeLOOP can achieve performant yet safe policies in a sample efficient manner. SafeLOOP reaches a higher reward than CPO, LBPO and PPO-lagrangian, while being orders of magnitude faster. SafeLOOP also achieves a policy with a lower cost faster than the baselines. 
From another aspect, the simulated RC-car experiments demonstrate the benefits of the terminal value function in safe RL. Figure~\ref{fig:safety_rc_car} shows the performance of LOOP, safeLOOP, PETS, and safePETS on this domain. PETS~\cite{chua2018deep} and safePETS do not consider a terminal value function. SafeLOOP is able to achieve high performance while maintaining the fewest constraint violations during training. Qualitatively, LOOP and safeLOOP are able to learn a safe drifting behavior, whereas PETS and safePETS fail to do so since drifting requires longer horizon reasoning beyond the fixed planning horizon in PETS. The results suggest that safeLOOP is a desirable choice of algorithm for safe RL due to its sample efficiency and the flexibility of incorporating constraints during deployment.

\vspace{-4mm}
\section{Conclusion} 
\label{sec:conclusion}
\vspace{-2mm}
In this work we analyze the H-step lookahead method under a learned model and value function and demonstrate empirically that it can lead to many benefits in deep reinforcement learning. We propose a framework LOOP which removes the computational overhead of trajectory optimization for value function update. We identify the actor-divergence issue in this framework and propose a modified trajectory optimization procedure - Actor Regularized Control. We show that the flexibility of H-step lookahead policy allows us to improve performance in online RL, offline RL as well as safe RL and this makes LOOP a strong choice of RL algorithm for robotic applications.

\newpage
\acknowledgments{We thank Tejus Gupta, Xingyu Lin and the members of R-PAD lab for insightful discussions. This material is based upon work supported by the United States Air Force and DARPA under Contract No. FA8750-18-C-0092, LG Electronics,
and the National Science Foundation under Grant No. IIS-1849154.}
\bibliography{main}
\newpage
\onecolumn

\appendix
\section{Theory}
\label{ap:theory}

\subsection{H-step lookahead with approximation error}
\label{ap:A1Hstep}

We aim to show that H-step model-based lookahead policies are more robust to certain types of approximation errors than 1-step greedy policies given an approximate value function. We restate Theorem 1 here for convenience and then provide a proof.

\begin{theoremr}
\label{thm:h_step_thm}
(H-step lookahead policy) Suppose $\hat{M}$ is an approximate dynamics model such that $\max_{s,a}D_{TV}\left(M(.|s,a),\hat{M}(.|s,a)\right)\le \epsilon_m$. Let $\hat{V}$ be an approximate value function such that $\max_s|V^*(s)-\hat{V}(s)|\le\epsilon_v$. Let the reward function by bounded in [0,$R_{\text{max}}$] and $\hat{V}$ be bounded in [0,$V_{\text{max}}$]. \reb{Let $\epsilon_p$ be the suboptimality incurred in H-step lookahead optimization (Eqn.~\ref{eq:H-step_objective}) such that $J^*-\hat{J}\le \epsilon_p$, where $J^*$ is the optimal return for the H-step optimization and $\hat{J}$ is the result of the suboptimal H-step optimization.} Then the performance of the H-step lookahead policy $\pi_{H,\hat{V}}$ can be bounded as:\\
\begin{equation*}
\label{eq:lookahead_bound}
     J^{\pi^*}-J^{\pi_{H,\hat{V}}}\le \frac{2}{1-\gamma^H}[C(\epsilon_m,H,\gamma)+\reb{\frac{\epsilon_p}{2}}+\gamma^H\epsilon_v]
\end{equation*}
where 
\begin{equation*}
    C(\epsilon_m,H,\gamma)=R_{\max} \sum_{t=0}^{H-1}\gamma^t t \epsilon_m + \gamma^H H\epsilon_mV_{\text{max}}
\end{equation*}

\end{theoremr}
\begin{proof}
Assume we have an $\epsilon_v$-approximate value function i.e $\|\hat{V}-V^*\|_\infty < \epsilon_v$ and we have an approximate transition model which satisfies $D_{TV}\left(M(.|s,a),\hat{M}(.|s,a)\right)\le \epsilon_m$ , similar to assumptions in~\cite{lowrey2018plan,efroni2020online}. We analyze the optimality gap of the policy which uses an H-step lookahead optimization (Eqn.~\ref{eq:H-step_objective}) with this approximate model and value function. First, we define some useful notations: let $\mathcal{M}$ be the MDP defined by $(\mathcal{S}, \mathcal{A}, M, r, s_0)$ which uses the ground truth dynamics $M$, state space $\mathcal{S}$, action space $\mathcal{A}$, reward function $r$ and starting state $s_0$, and let $\hat{\mathcal{M}}$ be the MDP defined by $(\mathcal{S}, \mathcal{A}, \hat{M}, r, s_0)$ which uses the approximate dynamics model $\hat{M}$. Correspondingly, let $\mathcal{H}$ be an H-step finite horizon MDP given by $(\mathcal{S}, \mathcal{A}, M, r_{\text{mix}}, s_0)$ and let $\hat{\mathcal{H}}$ be an H-step finite horizon MDP given by $(\mathcal{S}, \mathcal{A}, \hat{M}, r_{\text{mix}}, s_0)$ where 

\begin{equation}
    r_{\text{mix}}(s_t,a_t) =  \begin{cases}
    r(s,a) &\text{if $t<H$}\\
   \hat{V}(s_H) &\text{if $t=H$}
\end{cases}\\
\end{equation}
We redefine $\pi_{H,\hat{V}}$ to be the policy obtained by repeatedly optimizing for the H-step lookahead objective (Eqn.~\ref{eq:H-step_objective}) in $\hat{\mathcal{H}}$ and acting for H steps in $\mathcal{M}$. We do not consider the MPC setting for simplicity in proof i.e. the policy does not perform any replanning after taking its initial actions. We will use $\pi^*_{\mathcal{K}}$ denote the optimal policy for some MDP $\mathcal{K}$. Let $\hat{\tau}$ denote an H-step trajectory sampled by running $\pi^*_{\hat{\mathcal{H}}}$ in $\mathcal{M}$  and similarly $\tau$ is used to denote an H-step trajectory sampled by running $\pi^*_{\mathcal{H}}$ in $\mathcal{M}$. Let $\tau^*$ denote the H-step trajectory sampled by running $\pi^*_{\mathcal{M}}$ in $\mathcal{M}$. Let $p_{\hat{\tau}}$, $p_\tau$ and $p_{\tau^*}$ be the corresponding trajectory distributions. The performance gap we want to upper bound is given by:
\begin{align}
&J^{\pi^*}-J^{\pi_{H,\hat{V}}} = V^*(s_0) - V^{\pi_{H,\hat{V}}} (s_0)\\
    &=\E{\tau^*\sim p_{\tau^*}}{\sum\gamma^t r(s_t,a_t) +\gamma^H V^*(s_H)} - \E{\hat{\tau}\sim p_{\hat{\tau}}}{\sum\gamma^tr(s_t,a_t) +\gamma^H V^{\pi_{H,\hat{V}}}(s_H)}\\
    &=\E{\tau^*\sim p_{\tau^*}}{\sum\gamma^tr(s_t,a_t) +\gamma^H V^*(s_H)} -\E{\hat{\tau}\sim p_{\hat{\tau}}}{\sum\gamma^tr(s_t,a_t) +\gamma^H V^*(s_H)}\\&+\E{\hat{\tau}\sim p_{\hat{\tau}}}{\sum\gamma^tr(s_t,a_t) +\gamma^H V^*(s_H)}- \E{\hat{\tau}\sim p_{\hat{\tau}}}{\sum\gamma^tr(s_t,a_t) +\gamma^H V^{\pi_{H,\hat{V}}}(s_H)}\\
    \label{eq:ineq0}
    &= \E{\tau^*\sim p_{\tau^*}}{\sum\gamma^tr(s_t,a_t) +\gamma^H V^*(s_H)} -\E{\hat{\tau}\sim p_{\hat{\tau}}}{\sum\gamma^tr(s_t,a_t)+\gamma^H V^*(s_H)}\\&+ \gamma^H\E{\hat{\tau}\sim p_{\hat{\tau}}}{V^*(s_H)-V^{\pi_{H,\hat{V}}}(s_H)}
\end{align}

Since we have $|V^*(s)-\hat{V}(s)|\le\epsilon_v~\forall s$, we can bound the following expressions:
\begin{align}
\label{eq:ineq1}
    &\E{\tau^*\sim p_{\tau^*}}{\sum\gamma^tr(s_t,a_t) +\gamma^H V^*(s_H)} \le \E{\tau^*\sim p_{\tau^*}}{\sum\gamma^tr(s_t,a_t) +\gamma^H \hat{V}(s_H)}+\gamma^H\epsilon_v\\
\label{eq:ineq2}
    &\E{\hat{\tau}\sim p_{\hat{\tau}}}{\sum\gamma^tr(s_t,a_t) +\gamma^H V^*(s_H)} \ge \E{\hat{\tau}\sim p_{\hat{\tau}}}{\sum\gamma^tr(s_t,a_t) +\gamma^H \hat{V}(s_H)}-\gamma^H \epsilon_v
\end{align}

Subtracting these two inequalities (\ref{eq:ineq1} and \ref{eq:ineq2}), we get:
\begin{align}
\label{eq:ineq3}
    \E{\tau^*\sim p_{\tau^*}}{\sum\gamma^tr(s_t,a_t) +\gamma^H V^*(s_H)} -\E{\hat{\tau}\sim p_{\hat{\tau}}}{\sum\gamma^tr(s_t,a_t) +\gamma^H V^*(s_H)} \\\le\E{\tau^*\sim p_{\tau^*}}{\sum\gamma^tr(s_t,a_t) +\gamma^H \hat{V}(s_H)} \nonumber -\E{\hat{\tau}\sim p_{\hat{\tau}}}{\sum\gamma^tr(s_t,a_t) +\gamma^H \hat{V}(s_H)} + 2\gamma^H\epsilon_v 
\end{align}

Substituting Eqn.~\ref{eq:ineq3} into Eqn.~\ref{eq:ineq0} we can bound the performance gap as follows:
    %  &=\E{\tau^*\sim p_{\tau^*}}{\sum\gamma^tr(s_t,a_t) +\gamma^H V^*(s_H)} - \E{\hat{\tau}\sim p_{\hat{\tau}}}{\sum\gamma^tr(s_t,a_t) +\gamma^H V^{\pi_{H,\hat{V}}}(s_H)}\nonumber\\
\begin{align}
&J^{\pi^*}-J^{\pi_{H,\hat{V}}}=V^*(s_0) - V^{\pi_{H,\hat{V}}} (s_0)\nonumber\\
     &\le \E{\tau^*\sim p_{\tau^*}}{\sum\gamma^tr(s_t,a_t) +\gamma^H \hat{V}(s_H)}  -\E{\hat{\tau}\sim p_{\hat{\tau}}}{\sum\gamma^tr(s_t,a_t) +\gamma^H \hat{V}(s_H)} \\&+ 2\gamma^H\epsilon_v \nonumber + \gamma^H\E{\hat{\tau}\sim p_{\hat{\tau}}}{V^*(s_H)-V^{\pi_{H,\hat{V}}}(s_H)}\\
     &= \E{\tau^*\sim p_{\tau^*}}{\sum\gamma^tr(s_t,a_t) +\gamma^H \hat{V}(s_H)}-\E{\tau \sim p_{\tau}}{\sum\gamma^tr(s_t,a_t) +\gamma^H \hat{V}(s_H)}\\&+\E{\tau\sim p_{\tau}}{\sum\gamma^tr(s_t,a_t) +\gamma^H \hat{V}(s_H)} \nonumber -\E{\hat{\tau}\sim p_{\hat{\tau}}}{\sum\gamma^tr(s_t,a_t) +\gamma^H \hat{V}(s_H)} \\&+ 2\gamma^H\epsilon_v + \gamma^H\E{\hat{\tau}\sim p_{\hat{\tau}}}{V^*(s_H)-V^{\pi_{H,\hat{V}}}(s_H)}\\   
     \label{eq:ineq8}
     &\le \E{\tau\sim p_{\tau}}{\sum\gamma^tr(s_t,a_t) +\gamma^H \hat{V}(s_H)}  -\E{\hat{\tau}\sim p_{\hat{\tau}}}{\sum\gamma^tr(s_t,a_t) +\gamma^H \hat{V}(s_H)} \\&+ 2\gamma^H\epsilon_v + \gamma^H\E{\hat{\tau}\sim p_{\hat{\tau}}}{V^*(s_H)-V^{\pi_{H,\hat{V}}}(s_H)}
\end{align}
The last step is due to the fact that $\tau$ is generated by the optimal action sequence in the \textit{ground-truth} H-step MDP $\mathcal{H}$ as defined earlier which implies that $\E{\tau^*\sim p_{\tau^*}}{\sum\gamma^tr(s_t,a_t) +\gamma^H \hat{V}(s_H)}\le\E{\tau \sim p_{\tau}}{\sum\gamma^tr(s_t,a_t) +\gamma^H \hat{V}(s_H)}$ 
% \dave{which implies that $x < y$}.

Now we aim to characterize the performance gap between an optimal policy of MDP $\hat{\mathcal{H}}$ , $\pi^*_{\hat{\mathcal{H}}}$, with the optimal policy of  MDP $\mathcal{H}$, $\pi^*_{\mathcal{H}}$, evaluating both in the ground truth MDP $\mathcal{H}$.  We wish to characterize this performance gap as a function of model errors and value errors $f(\epsilon_m,\epsilon_v,\gamma, H)$:. 
\begin{align*}
    \E{\tau\sim p_{\tau}}{\sum\gamma^tr(s_t,a_t) +\gamma^H \hat{V}(s_H)}  -\E{\hat{\tau}\sim p_{\hat{\tau}}}{\sum\gamma^tr(s_t,a_t) +\gamma^H \hat{V}(s_H)}\le f(\epsilon_m,\epsilon_v,\gamma, H)
\end{align*}
Let $J^{\pi}_{\mathcal{H}}$ denote the performance of policy $\pi$ when evaluated in MDP ${\mathcal{H}}$ starting from same initial state $s_0$. Then we can write this performance gap as
\begin{align}
    & \E{\tau\sim p_{\tau}}{\sum\gamma^t r(s_t,a_t) +\gamma^H \hat{V}(s_H)}  -\E{\hat{\tau}\sim p_{\hat{\tau}}}{\sum\gamma^t r(s_t,a_t) +\gamma^H \hat{V}(s_H)} \label{eq:not_recurrence_term}\\
    &= J_{\mathcal{H}}^{\pi^*_{\mathcal{H}}} - J_{\mathcal{H}}^{\pi^*_{\hat{\mathcal{H}}}}\\
    &= J_{\mathcal{H}}^{\pi^*_{\mathcal{H}}} - J_{\hat{\mathcal{H}}}^{\pi^*_{\mathcal{H}}}+  J_{\hat{\mathcal{H}}}^{\pi^*_{\mathcal{H}}} -J_{\hat{\mathcal{H}}}^{\pi^*_{\hat{\mathcal{H}}}}+J_{\hat{\mathcal{H}}}^{\pi^*_{\hat{\mathcal{H}}}}   - J_{\mathcal{H}}^{\pi^*_{\hat{\mathcal{H}}}}\\
    &= \left( J_{\mathcal{H}}^{\pi^*_{\mathcal{H}}} - J_{\hat{\mathcal{H}}}^{\pi^*_{\mathcal{H}}}\right) - \left(J_{\mathcal{H}}^{\pi^*_{\hat{{\mathcal{H}}}}}-J_{\hat{{\mathcal{H}}}}^{\pi^*_{\hat{\mathcal{H}}}}\right) + \left(J_{\hat{\mathcal{H}}}^{\pi^*_{\mathcal{H}}} - J_{\hat{\mathcal{H}}}^{\pi^*_{\hat{\mathcal{H}}}} \right)\\
    &\le \left( J_{\mathcal{H}}^{\pi^*_{\mathcal{H}}} - J_{\hat{\mathcal{H}}}^{\pi^*_{\mathcal{H}}}\right) - \left(J_{\mathcal{H}}^{\pi^*_{\hat{{\mathcal{H}}}}}-J_{\hat{{\mathcal{H}}}}^{\pi^*_{\hat{{\mathcal{H}}}}}\right)\reb{+\epsilon_p}\\
    \label{eq:ineq4}
    &\le 2 \max_{\pi\in \{\pi^*_{\mathcal{H}},\pi^*_{\hat{\mathcal{H}}}\}}| \left( J_{\mathcal{H}}^{\pi} - J_{\hat{{\mathcal{H}}}}^{\pi}\right)| \reb{+\epsilon_p}
\end{align}

\reb{The second-to-last equation is due to the assumed suboptimality of H-step lookahead planner where we have $\forall$ policies $\pi$, $J_{\hat{{\mathcal{H}}}}^{\pi^*_{\hat{{\mathcal{H}}}}}+\epsilon_p\ge J_{\hat{{\mathcal{H}}}}^\pi$ .}
% \dave{What is $\pi^*_{\mathcal{H}}$?  I don't recall that being defined earlier.}
Since the total variation between $M$ and $\hat{M}$ is at most $\epsilon_m$, i.e $D_{TV}\left(M(.|s,a),\hat{M}(.|s,a)\right)\le \epsilon_m$, we have that $|\rho_1^t(s,a)-\rho_2^t(s,a)|\le t\epsilon_m$, where $\rho_1(s,a)$ is the discounted state-action visitation induced by $\pi$ on $\mathcal{H}$, $\rho_2(s,a)$ is the discounted state-action visitation induced by the same policy on $\hat{\mathcal{H}}$ and superscript $t$ indicates the state-action marginal at the $t^{th}$ timestep (for proof see Lemma B.2 Markov Chain TVD Bound~\cite{janner2019trust}).  Then we can write the performance of policy $\pi$ in terms of its induced state marginal and the reward function, i.e $J_{\mathcal{H}}^\pi=\sum_{s,a}\rho_1(s,a) r_{\text{mix}}(s,a) =\sum_{s,a}\sum_{t=0}^H\gamma^t\rho_1^t(s,a)r_{\text{mix}}(s,a)$ and use the Markov chain TVD bound:
\begin{align}
J_{\mathcal{H}}^\pi - J_{\hat{{\mathcal{H}}}}^\pi &= \sum_{s,a} (\rho_1(s,a)-\rho_2(s,a))r_{\text{mix}}(s,a)\\
|J_{\mathcal{H}}^\pi - J_{\hat{{\mathcal{H}}}}^\pi| &= |\sum_{s,a} (\rho_1(s,a)-\rho_2(s,a))r_{\text{mix}}(s,a)|\\
&= |\sum_{s,a}\sum_{t=0}^{H}\gamma^t (\rho_1^t(s,a)-\rho_2^t(s,a))r_{\text{mix}}^t(s,a)|\\
&\le \sum_{s,a}\sum_{t=0}^{H}\gamma^t |(\rho_1^t(s,a)-\rho_2^t(s,a))|r_{\text{mix}}^t(s,a)\label{eq:update_eq}\\
&\le R_{\max} \sum_{t=0}^{H-1} \gamma^t t \epsilon_m + \gamma^H H\epsilon_m V_{\text{max}} \\
\label{eq:ineq5}
&= C(\epsilon_m,H,\gamma)
\end{align}
% Epsilon is the worst error in h-horizon model rollout
Combining Eqn.~\ref{eq:ineq4} and Eqn.~\ref{eq:ineq5} we have:
\begin{equation}
\label{eq:ineq6}
   \E{\tau\sim p_{\tau}}{\sum\gamma^t r(s_t,a_t) +\gamma^H \hat{V}(s_H)}  -\E{\hat{\tau}\sim p_{\hat{\tau}}}{\sum\gamma^t r(s_t,a_t) +\gamma^H \hat{V}(s_H)} \le 2C(\epsilon_m,H,\gamma)\reb{+\epsilon_p}
\end{equation}

We substitute Eqn.~\ref{eq:ineq6} in Eqn.~\ref{eq:ineq8}. Also observe that the last term in Eqn.~\ref{eq:ineq8} $\gamma^H\E{\hat{\tau}}{V^*(s_H)-V^\mpc(s_H)}$  can be bounded recursively. Then, we will have the following optimality gap for the H-step lookahead policy $\mpc$:

\begin{equation}
    J^{\pi^*}-J^{{\pi_{H,\hat{V}}}}\le \frac{2}{1-\gamma^H}[C(\epsilon_m,H,\gamma)\reb{+\frac{\epsilon_p}{2}}+\gamma^H\epsilon_v]
\end{equation}
The H-step lookahead policy $\pi_{H,\hat{V}}$ reduces the dependency on $\epsilon_v$ (the maximum error of the value function) by a factor of $\gamma^H$ and introduces an additional dependency on $\epsilon_m$ (the maximum error of the model). In contrast, when we use 1-step greedy policy, the performance gap is bounded by (Lemma~\ref{lemma:greedy_lookahead}):
\begin{equation}
    J^{\pi^*}-J^{{\pi_{H,\hat{V}}}}\le \frac{\gamma}{1-\gamma}[2\epsilon_v]
\end{equation}
Lemma~\ref{lemma:greedy_lookahead} can be seen as a special case of our bound when $\epsilon_m$ is set to 0 and $H$ is set to 1.
\end{proof}

\subsection{H-step lookahead with model generalization error}

In this section, we derive a similar proof as the previous section with a weaker assumption on model error. We consider a model trained by supervised learning where the sample error can be computed by PAC generalization bounds which bounds the expected loss and empirical loss under a dataset with high probability.

We define $D$ to be the dataset of transitions and $\pi_D$ to be the data collecting policy.

\begin{corollary}
\label{thm:h_step_model_gen}
(H-step lookahead with function approximation) Suppose $\hat{M}$ is an approximate dynamics model such that $\max_t \E{s\sim\pi_{D,t}}{D_{TV}(M(.|s,a)\|\hat{M}(.|s,a))}\le\tilde{\epsilon}_m$. Let $\hat{V}$ be an approximate value function such that $\max_s|V^*(s)-\hat{V}(s)|\le\epsilon_v$. Let the maximum TV distance of state distribution visited by lookahead policy $\pi_{H,\hat{V}}$ be bounded wrt state visitation of data generating policy by $\max_t \E{s\sim \pi_{H,\hat{V}}}{D_{TV}(\rho^t_{\pi_{H,\hat{V}}}\|\rho^t_{\pi_D})}\le \epsilon_i$ and $\max \left(D_{TV}(\pi_D(a|s)||\pi^*_{H}(a|s)),D_{TV}(\pi_D(a|s)||\pi^*_{\hat{H}}(a|s))\right) \le \tilde{\epsilon_\pi}~~\forall s$. Let the reward function by bounded in [0,$R_{\text{max}}$] and $\hat{V}$ be bounded in [0,$V_{\text{max}}$]. Then the performance of the H-step lookahead policy $\pi_{H,\hat{V}}$ can be bounded as:\\
\begin{equation*}
\label{eq:lookahead_bound}
     J^{\pi^*}-J^{\pi_{H,\hat{V}}}\le \frac{2}{1-\gamma^H}[C(\tilde{\epsilon_m},\tilde{\epsilon_\pi},\epsilon_i,H,\gamma)+\gamma^H\epsilon_v]
\end{equation*}
where 
\begin{equation*}
    C(\tilde{\epsilon_m},\tilde{\epsilon_\pi},H,\gamma) =  R_{\max} \sum_{t=0}^{H-1} \gamma^t t (\tilde{\epsilon_m}+\tilde{\epsilon_\pi}) +R_{\max}\epsilon_i+ \gamma^H H(\tilde{\epsilon_m}+\tilde{\epsilon_\pi}) V_{\text{max}}
\end{equation*}
\end{corollary}

\begin{proof}
 In the function approximation setting, a more realistic perfomance bound depends on the generalization error of model and distribution shift for the new policy under the collected dataset of transitions $D$. Let $\pi_D$ be the data collecting policy. Let $ \E{s\sim\pi_{D,t}}{D_{TV}(M(.|s,a)||\hat{M}(.|s,a))}\le\tilde{\epsilon}_m~~\forall s$ and $\max \left(D_{TV}(\pi_D(a|s)||\pi^*_{H}(a|s)),D_{TV}(\pi_D(a|s)||\pi^*_{\hat{H}}(a|s))\right) \le \tilde{\epsilon_\pi}~~\forall s$. Following Lemma B.2 Markov Chain TVD Bound~\cite{janner2019trust} with model generalization error $\tilde{\epsilon}_m$, policy distribution shift $\tilde{\epsilon}_\pi$ and bounded state visitation of lookahead policy by $\epsilon_i$,  we have: $|\rho_1^t(s,a)-\rho_2^t(s,a)|\le t(\tilde{\epsilon}_m+\tilde{\epsilon}_\pi)+\epsilon_i$
 Substituting the new state-action divergence bound in Eqn.~\ref{eq:update_eq} from Theorem~\ref{thm:h_step_thm} we get the following performance bound:
\begin{equation}
     J^{\pi^*}-J^{\pi_{H,\hat{V}}}\le \frac{2}{1-\gamma^H}[C(\tilde{\epsilon_m},\tilde{\epsilon_\pi},\epsilon_i,H,\gamma)+\gamma^H\epsilon_v]
\end{equation}
where $C(\tilde{\epsilon_m},\tilde{\epsilon_\pi},H,\gamma) =  R_{\max} \sum_{t=0}^{H-1} \gamma^t t (\tilde{\epsilon_m}+\tilde{\epsilon_\pi}) +R_{\max}\epsilon_i+ \gamma^H H(\tilde{\epsilon_m}+\tilde{\epsilon_\pi}) V_{\text{max}}$. 

Intuitively this bound highlights the tradeoff between model error and value error reasonably when the dataset is sufficiently exploratory to cover $\pi_H^*$ and H-step lookahead policy has visitation close to the dataset.

\end{proof}

\subsection{H-step lookahead with Empirical Dataset Distribution using Fitted-Q Iteration}

In this section, we take a look at the analysis of H-step lookahead under a set of different assumptions. In particular, we assume a form of model generalization error and that the optimal H-step trajectory is obtained via fitted-Q iteration in the H-step MDP at every timestep during policy deployment. This analysis largely follows the fitted-Q iteration analysis from~\cite{munos2005error,munos2008finite,agarwal2019reinforcement} but we adapt it to H-step lookahead in a simplified form.

\begin{assumption}
\label{assump1}
Let our replay buffer dataset be denoted by $D$ and the data generating distribution be given by $d^{\pi_D}$, where $\pi_D$ is the data generating policy. Let the Q-function class is given by $\mathcal{Q}\subset \mathbb{R}^{S\times A}$. The empirical bellman update $\hat{\mathcal{T}}Q$  under the learned model is given by:
\begin{equation}
L_{d^{^{\pi_{\hat{M}}}}}(Q,Q^k) = \E{s,a,r,s'\sim d^{^{\pi_{\hat{M}}}}}{(Q(s,a)-r-\gamma Q^k(s',\pi_Q(s')))^2}
\end{equation}
where $Q^k$ is the Q-function at k iteration, $d^{\pi_{\hat{M}}}$ is the state visitation under a learned model $\hat{M}$ from dataset $D$. Also we define:
\begin{equation}
L_{d^{\pi_D}}(Q,Q^k) = \E{s,a,r,s'\sim d^{\pi_D}}{(Q(s,a)-r-\gamma Q^k(s',\pi_Q(s')))^2}
\end{equation}
A form of model generalization error: We assume the following uniform deviation bound which holds with high probability ($\ge 1-\delta$):
\begin{equation}
    \forall Q,Q^k,~|L_D(Q,Q^k)-L_{d^{\pi_D}}(Q,Q^k)|\le \tilde{\epsilon}_m 
\end{equation}
This bound can be obtained by concentration inequality as in~\cite{agarwal2019reinforcement} using concentration inequality $\tilde{\epsilon}_m$ to be a function of size of dataset $|D|$, $\delta$ and size of function space for $\mathcal{Q}$.
\end{assumption}
Intuitively the assumption above states that the bellman error obtained in the data-generating distribution is close to the bellman error obtained via state-action distribution induced by the learned model, where the model is learned on a finite fixed dataset $D$ sampled from data generating distribution.

In the following analysis, we assume that H-step lookahead policy is obtained by performing fitted-Q iteration in the H-step approximate MDP $\hat{\mathcal{H}}$ defined in Theorem~\ref{thm:h_step_thm}.

\begin{theorem}
\label{thm:h_step_fitted_q}
Suppose $\hat{M}$ is an approximate dynamics model such that Assumption~\ref{assump1} holds. Let $\hat{V}$ be an approximate value function such that $\max_s|V^*(s)-\hat{V}(s)|\le\epsilon_v$. Let the reward function by bounded in [0,$R_{\text{max}}$] and $\hat{V}$ be bounded in [0,$V_{\text{max}}$]. Let concentrability coefficient $\tilde{C}$ be such that $\forall s,a~\frac{\nu(s,a)}{d^{\pi_D}(s,a)}\le \tilde{C}$ where $\nu(s,a)$ is state-action distribution induced by any non-stationary policy. Then the performance of the H-step lookahead policy $\pi_{H,\hat{V}}$ obtained by running fitted-Q iteration on the learned model to convergence can be bounded as:\\
\begin{equation*}
\label{eq:lookahead_bound}
     J^{\pi^*}-J^{\pi_{H,\hat{V}}}\le \frac{2}{1-\gamma^H}[C(\tilde{\epsilon}_m,\tilde{C},H,\gamma)+\gamma^H\epsilon_v]
\end{equation*}
where 
\begin{equation*}
    C(\tilde{\epsilon}_m,\tilde{C},H,\gamma)=\frac{2(1-\gamma^H)}{1-\gamma}\left(\frac{1}{1-\gamma}\sqrt{2\tilde{\epsilon}_m\tilde{C}}\right)
\end{equation*}

\end{theorem}

\begin{proof}

In this section we analyze the performance of H-step lookahead policies under the assumptions for Fitted Q Iteration~\cite{agarwal2019reinforcement}. This analysis extends the fitted-Q iteration analysis from greedy to H-step lookahead policies.

Let $\|g\|_{p,\nu}$ denote a weighted p-norm under distribution $\nu$ given by $\|g\|_{p,\nu}=\E{s\sim\nu}{|g(s)|^p}^{\frac{1}{p}}$. We start by reusing the previous analysis in Theorem~\ref{thm:h_step_thm} under the new stated assumptions to replace the bound for Eqn.~\ref{eq:not_recurrence_term}. Let $\pi^*_{\hat{\mathcal{H}}}$ be denoted by $\hat{\pi}_H$ and $\pi^*_{\mathcal{H}}$ by $\pi^*_H$ for ease of notation. In this analysis $\hat{\pi}_H$ is the 1-step greedy policy obtained from $Q_k$ the learned Q-function after k iterations of fitted-Q iteration on the H-step MDP $\hat{\mathcal{H}}$.

Rewriting Eqn.~\ref{eq:not_recurrence_term}:
\begin{align}
    & \E{\tau\sim p_{\tau}}{\sum\gamma^t r(s_t,a_t) +\gamma^H \hat{V}(s_H)}  -\E{\hat{\tau}\sim p_{\hat{\tau}}}{\sum\gamma^t r(s_t,a_t) +\gamma^H \hat{V}(s_H)} \\
    &= J_{\mathcal{H}}^{\pi^*_H} - J_{\mathcal{H}}^{\hat{\pi}_H}
\end{align}
Using performance difference lemma we can write:
\begin{align}
J_{\mathcal{H}}^{\pi^*_H} - J_{\mathcal{H}}^{\hat{\pi}_H} &\le \sum_{t=1}^H \gamma^{t-1}\E{s\sim d^{\hat{\pi}_H}}{V^{\pi^*_H}(s)-Q^{\pi^*_H}(s,\hat{\pi}_H)}\\
&\le \sum_{t=1}^H \gamma^{t-1}\E{s\sim d^{\hat{\pi}_H}}{V^{\pi^*_H}(s)-Q_k(s,\pi^*_H)+Q_k(s,\hat{\pi}_H)-Q^{\pi^*_H}(s,\hat{\pi}_H)}\\
&\le   \sum_{t=1}^H \gamma^{t-1} \left(\|Q^{\pi^*_H}-Q_k\|_{1,d^{\hat{\pi}_H}\times\pi^*_H}  + \|Q^{\pi^*_H}-Q_k\|_{1,d^{\hat{\pi}_H}\times\hat{\pi}_H}\right)\\
&\le   \sum_{t=1}^H \gamma^{t-1} \left(\|Q^{\pi^*_H}-Q_k\|_{d^{\hat{\pi}_H}\times\pi^*_H}  + \|Q^{\pi^*_H}-Q_k\|_{d^{\hat{\pi}_H}\times\hat{\pi}_H}\right)
\end{align}

The second line follows from the fact that $Q_k(s,\hat{\pi}_H)\ge Q_k(s,\pi^*_H)$ since $\hat{\pi}_H$ maximizes $Q_k$.The concentrability assumptions allows us to compare weighted norms under state distribution induced by any policy $\nu(s,a)$ and $d^{\pi_D}(s,a)$ as follows: $\|.\|_\nu\le \sqrt{\tilde{C}}\|.\|_{d^{\pi_D}}$. We can bound $\|Q^{\pi^*_H}-Q_k\|_{\mu,\pi}$ for arbitrary state distribution $\mu$ and policy $\pi$ as:
\begin{align}
    \|Q^{\pi^*_H}-Q_k\|_{\mu\times \pi} &= \|Q^{\pi^*_H}-\mathcal{T}Q_{k-1}+\mathcal{T}Q_{k-1}-Q_k\|\\
    &\le \|\mathcal{T}Q^{\pi^*_H}-\mathcal{T}Q_{k-1}\|_{\mu\times\pi}+ \|\mathcal{T}Q_{k-1}-Q_k\|_{\mu\times\pi}\\
    &\le \|\mathcal{T}Q^{\pi^*_H}-\mathcal{T}Q_{k-1}\|_{\mu\times\pi}+\sqrt{\tilde{C}} \|\mathcal{T}Q_{k-1}-Q_k\|_{d^{\pi_D}}\\
    &= \gamma \|Q_{k-1}(\cdot,\pi_{Q_{k-1}})-Q^{\pi^*_H}(\cdot,\pi^*_H)\|_{P(\mu\times\pi)}+\sqrt{\tilde{C}} \|\mathcal{T}Q_{k-1}-Q_k\|_{d^{\pi_D}}
\end{align}
where $P(\mu\times\pi)$ as distribution over $\mathcal{S}$ where $s,a\sim\mu,~s'\sim p(s,a)$.
Define $\pi_{mix}=\argmax_{a\in\mathcal{A}}(Q^{\pi^*_H}(s,a), Q_{k-1}(s,a))$. Then we have:
\begin{align}
    \|Q^{\pi^*_H}-Q_k\|_{\nu,\pi} &= \gamma \|Q_{k-1}(.,\pi_{Q_{k-1}})-Q^{\pi^*_H}(.,\pi^*_H)\|_{P(\mu\times\pi)}+\sqrt{|\mathcal{A}|\tilde{C}} \|\mathcal{T}Q_{k-1}-Q_k\|_{d^{\pi_D}}\\
    &\le \sqrt{\tilde{C}} \|\mathcal{T}Q_{k-1}-Q_k\|_{d^{\pi_D}} + \gamma \|Q_{k-1}-Q^{\pi^*_H}\|_{P(\mu\times\pi) \times\pi_{mix}}
\end{align}
The second term $\|Q_{k-1}-Q^{\pi^*_H}\|_{P(\mu\times\pi) \times\pi_{mix}}$ can be expanded via recursion for k times, since the same analysis holds. We now bound $\|\mathcal{T}Q_{k-1}-Q_k\|_{d^{\pi_D}} $.
\begin{align}
    \|\mathcal{T}Q_{k-1}-Q_k\|^2_{d^{\pi_D}} &= L_{d^{\pi_D}}(Q_k,Q_{k-1})- L_{d^{\pi_D}}(\mathcal{T}Q_{k-1},Q_{k-1})\\
    &\le L_{D}(Q_k,Q_{k-1})- L_{D}(\mathcal{T}Q_{k-1},Q_{k-1})+2\tilde{\epsilon}_m~~\text{w.p}\ge1-\delta\\
    &\le 2\tilde{\epsilon}_m
\end{align}
As fitted Q iteration converges $k\rightarrow\infty$ for $\gamma<1$, we have:
\begin{align}
    \|Q^{\pi_H^*}-Q_k\|_{\mu \times\pi}\le \frac{1-\gamma^k}{1-\gamma}\sqrt{2\tilde{\epsilon}_m\tilde{C}}+\gamma^k\frac{V_{max}}{1-\gamma}
\end{align}

 In this analysis we obtain $\hat{\pi}_H$ by performing fitted Q iteration ($k\rightarrow \infty$) under the dataset $D$. Therefore our bound for Eqn.~\ref{eq:not_recurrence_term} from the previous analysis under the current assumptions reduces to:
\begin{align}
    & \E{\tau\sim p_{\tau}}{\sum\gamma^t r(s_t,a_t) +\gamma^H \hat{V}(s_H)}  -\E{\hat{\tau}\sim p_{\hat{\tau}}}{\sum\gamma^t r(s_t,a_t) +\gamma^H \hat{V}(s_H)} \\
    &= J_{\mathcal{H}}^{\pi^*_H} - J_{\mathcal{H}}^{\hat{\pi}_H}\\
    &\le \frac{2(1-\gamma^H)}{1-\gamma}\left(\frac{1}{1-\gamma}\sqrt{2\tilde{\epsilon}_m\tilde{C}}\right)\\
    &\le C(\tilde{\epsilon}_m,\tilde{C},H,\gamma)
\end{align}

Plugging this back in our previous analysis we have the following performance bound for H-step lookahead policy:

\begin{equation}
    J^{\pi^*}-J^{{\pi_{H,\hat{V}}}}\le \frac{2}{1-\gamma^H}[C(\tilde{\epsilon}_m,\tilde{C},H,\gamma)+\gamma^H\epsilon_v]
\end{equation}
where $C(\tilde{\epsilon}_m,\tilde{C},H,\gamma)= \frac{2(1-\gamma^H)}{1-\gamma}\left(\frac{1}{1-\gamma}\sqrt{2\tilde{\epsilon}_m\tilde{C}}\right)$.
\end{proof}
\subsection{ARC constrains trajectories close to the parameterized actor}
\label{ap:lemma2_proof}
In section~\ref{sec:actor_divergence}, we use ARC, an iterative importance sampling procedure to solve the constrained optimization in Eqn.~\ref{eq:ARC}. The following lemma shows that the final trajectory distribution output as a result of finite importance sampling iteration is bounded in total variation to the trajectory distribution given by the parameterized actor.

\begin{lemma}
\label{thm:implicit_kl}
Let $p^\tau_{prior}$ be a distribution over action sequences. Applying M KL-based trust region steps of size $\epsilon$ to $p^\tau_{prior}$ results in a distribution $p^\tau_{M}$ that satisfies:
\begin{equation}
    D_{TV}(p^\tau_{prior}||p^\tau)\le T\sqrt{\frac{\epsilon}{2}}
\end{equation}
\end{lemma}
\begin{proof}
This lemma is adapted from~\cite{matsushima2020deployment} and provided for completeness. Let $p^\tau_{k}$ be the distribution at the k trust region step. 
$p^\tau_{0}=p^\tau_{prior}$
Using Pinsker's inequality we have:
\begin{align}
      D_{KL}(p^\tau_{k}||p^\tau_{k+1})&\le \epsilon\\
      D_{TV}(p^\tau_{k}||p^\tau_{k+1})&\le \sqrt{\frac{\epsilon}{2}}
\end{align}
 Using triangle inequality we have:
 \begin{align}
      D_{TV}(p^\tau_{prior}||p^\tau_{M})&\le M\sqrt{\frac{\epsilon}{2}}
\end{align}

\end{proof}

\newpage
\section{Algorithm Details}
\label{ap:algorithm_details}

\subsection{LOOP for online RL}

\vspace{0.5cm}
\begin{algorithm}
  Initialize the parametrized actor $\pi_{\phi}$, Q-function $Q_{\theta}$, predictive model $\hat{M}_{\psi}$, empty replay buffer $D$.
  Given planning horizon H.
  \caption{LOOP-SAC (for Online RL and Safe RL)}
  \begin{algorithmic}[1]
    \State // Training
    \ForEach{$t=1..(train\_steps)$}%
    \State Select action given by $a=ARC(s,
    \pi_\phi)$. \algorithmiccomment{Use safeARC for safeLOOP}
    \State Execute $a$ in the environment and observe reward $r$ and new state $s'$.
    \State Store the transition $(s,a,r,s')$ in replay buffer $D$.
    \State Optimize  $\pi_\phi$ and $Q_\theta$ using SAC over replay buffer $D$.
    \State Train model $\hat{M}_\psi$ on the replay buffer $D$ until convergence every $K_m$ training steps.
    \EndForEach
    \State // Evaluation
    \ForEach{$t=1..(eval\_steps)$}%
    \State Select action given by $a=ARC(s,
    \pi_\phi)$.
    \State Execute $a$ in the environment and observe reward $r$ and new state $s'$.
    \EndForEach
  \end{algorithmic}
  \label{alg:loop_algo}
\end{algorithm}

\vspace{0.5cm}
\subsection{LOOP for offline RL}

\vspace{0.5cm}
\begin{algorithm}
  Initialize the parametrized actor $\pi_{\phi}$, Q-function $Q_{\theta}$, predictive model $\hat{M}_{\psi}$, empty replay buffer $D$.
  Given planning horizon H.
  \caption{LOOP-offline}
  \begin{algorithmic}[1]
    \State // Training
    \State Train model $\hat{M}_\psi$ on the replay buffer $D$ till convergence.
    \State Run an Offline RL algorithm till convergence on $D$ to learn $Q_{\theta}$ and $\pi_{\phi}$.
    \State // Evaluation
    \ForEach{$t=1..(eval\_steps)$}%
    \State Select action given by $a=ARC(s_t,
    \pi_\phi)$.
    \State Execute $a$ in the environment and observe reward $r$ and new state $s'$.
    \EndForEach
  \end{algorithmic}
  \label{alg:loop_offline_algo}
\end{algorithm}

\newpage
\subsection{Actor Regularized Control (ARC)}
\label{ap:arc_algo}
Eqn.~\ref{eq:ARC} gives a general constrained optimization for policy update. In Eqn.~\ref{eq:ARC}, with terminal state-action value functions,when $[p^\tau_{prior}=\text{Uniform} ~,H=0]$, we recover the SAC~\cite{haarnoja2018soft} deployment policy, when $[p^\tau_{prior}=\pi^\beta~,H=0]$, we recover the AWAC~\cite{nair2020accelerating} deployment policy and when $\pi_{prior}=\mathcal{N}(0,\sigma)$,  we recover the MPPI~\cite{williams2016aggressive} deployment policy.

In the LOOP framework we use ARC as our trajectory optimization routine to solve Eqn.~\ref{eq:ARC}. Algorithm~\ref{alg:ARC} shows the pseudocode for ARC routine used for Online and Offline RL experiments.
\\

\begin{algorithm}
  Input: $s_{T},~\pi_\phi$\\
  Given the parameterized actor $\pi_{\phi}$, Q-function $Q_{\theta}$, predictive model $\hat{M}_{\psi}$, reward model $\hat{r}$, replay buffer $D$, Planning Horizon H,  1-timestep shifted solution from the previous timestep $\mu^{T-1}$, ARC iterations $n_{ARC}$, number of trajectories (population size) $N$.
  \caption{Actor Regularized Control (ARC)}
  \begin{algorithmic}[1]
    \ForEach{$i=1..n_{ARC}$}%
    \State $\textbf{R}_{1:N}=0$ \algorithmiccomment{Rewards of N trajectories}
    \State $\textbf{A}_{1:N,1:H}=0$ \algorithmiccomment{N action sequences with horizon H}
        \ForEach{$j=1..N \text{ trajectories}$}%
        % \ForEach{$p=1..P \text{ particles}$}%
        \State // Generate a trajectory with the model
        \State $s_{1} = s_{T}$
        \ForEach{$t=1..H \text{ horizon}$}%
        % \State $\epsilon\sim \mathcal{N}(0,\sigma)$\algorithmiccomment{Sample noise}
            \State // Generate actions from a mixture prior
            \State$\textbf{A}_{j,t}= a_t = \beta \pi_\phi(s_t)+(1-\beta)\mathcal{N}(\mu^{T-1}_{t},\sigma)$
            % \State\wenxuan{$\textbf{A}_{j,t}= a_t = \beta \pi_\phi(s_t)+(1-\beta)\mathcal{N}(?,?)$}
            \State $s_{t+1} =\hat{M}_\psi(s_t,a_t)$ 
        \EndForEach
        \State // Rollout the action sequence $P$ times in each model within the ensemble
        \State $R=0$
        \ForEach{$k=1..K \text{models}$}%
            \ForEach{$p=1..P \text{ particles}$}%
                \State $s_{1} = s_{T}$
                \ForEach{$t=1..H \text{ horizon}$}%
                % \State $\epsilon\sim \mathcal{N}(0,\sigma)$\algorithmiccomment{Sample noise}
                    \State $a_t = \textbf{A}_{j,t}$
                    \State $s_{t+1} =\hat{M}_\psi(s_t,a_t)$
                    \State $R = R + \gamma^{t-1}(\mathbbm{1}(t=H) Q_\theta(s_t,a_t)+\mathbbm{1}(t\neq H) \hat{r}(s_t,a_t))$
                \EndForEach
            \EndForEach
        \EndForEach
        \State //Uncertainty penalized average reward
        \State $\textbf{R}_j = \frac{1}{K} \left(\sum_{k=1}^K(R/P) - \beta_{pess}
        \sum(R/P-\sum_{k=1}^K(\frac{R}{KP}))^2\right)$
    \EndForEach
    \State $\mu_{new,1:H}$ = \text{weighted-mean}($\textbf{A}_{1:N}$, weights = exp($\textbf{R}_{1:N}/\eta$))
    \State $\Sigma_{new,1:H}$ = \text{weighted-mean}($(\textbf{A}_{1:N}-\mu_{new})^2$, weights = exp($\textbf{R}_{1:N}/\eta$))  
 \State $\mu^T_{i+1} =  \ \alpha * \mu_{new} + (1-\alpha)\mu^T_{i}$ \algorithmiccomment{Update mean}
    \State $\Sigma^T_{i+1} =  \ \alpha * \Sigma_{new} + (1-\alpha)\Sigma^T_{i}$ \algorithmiccomment{Update variance}
    \EndForEach
  \end{algorithmic}
  Output: $\mu^T=\mu^T_{n_{ARC}+1}$
  \label{alg:ARC}
\end{algorithm}

\vspace{0.5cm}
$\beta_{pess}$ is set to zero for Online RL experiments and safe RL experiments where trajectories are scored by unpenalized average. It is tuned for Offline RL experiments as detailed in Appendix~\ref{ap:exp_details}.

\newpage
\subsubsection{ARC for safe-RL}
\label{ap:safeARC_algo}
We optimize for the following objective in LOOP for safe RL:

\begin{equation}
    \text{argmax}_{a_t} \E{\hat{M}}{R_{H,\hat{V}}(s_t)}
    \text{s.t. } \max_{[K]} \sum_{t=0}^H\gamma^tc(s_t,a_T)\le d_0
\end{equation}
where $[K]$ are the model ensembles, $c$ is the constraint cost function and $R_{H,\hat{V}}$ is the H-horizon lookahead objective defined in Eqn.~\ref{eq:H-step_objective}. We incorporate safety in the trajectory optimization procedure following previous work~\cite{liu2020safe,wen2020constrained}. The pseudocode for safeARC used in safeLOOP is shown in Algorithm~\ref{alg:safeARC}.
\\
\begin{algorithm}[h]
%   \wenxuan{Need to match the changes in Algorithm 3.}
  Input: $s_{T},~\pi_\phi$\\
  Given the parameterized actor $\pi_{\phi}$, Q-function $Q_{\theta}$, predictive model $\hat{M}_{\psi}$, reward model $\hat{r}$, replay buffer $D$, planning horizon H, 1 timestep shifted solution from the previous timestep $\mu^{T-1}$, safety threshold $d_0$, minimal safe trajectories $m$, ARC iterations $n_{ARC}$, number of trajectories (population size) $N$.
  \caption{safeARC}
  \begin{algorithmic}[1]
    \ForEach{$i=1..n_{ARC}$}%
    \State $\textbf{R}_{1:N}=0$ \algorithmiccomment{Rewards of N trajectories}
    \State $\textbf{C}_{1:N}=0$ \algorithmiccomment{Cost of N trajectories}
    \State $\textbf{A}_{1:N,1:H}=0$ \algorithmiccomment{N action sequences with horizon H}
    \ForEach{$j=1..N \text{ trajectories}$}%
        % \ForEach{$p=1..P \text{ particles}$}%
        \State // Generate a trajectory with the model
        \State $s_{1} = s_{T}$
        \ForEach{$t=1..H \text{ horizon}$}%
        % \State $\epsilon\sim \mathcal{N}(0,\sigma)$\algorithmiccomment{Sample noise}
            \State // Generate actions from a mixture prior
            \State$\textbf{A}_{j,t}= a_t = \beta \pi_\phi(s_t)+(1-\beta)\mathcal{N}(\mu^{T-1}_{t},\sigma)$
            \State $s_{t+1} =\hat{M}_\psi(s_t,a_t)$ 
        \EndForEach
        \State // Rollout the action sequence $P$ times in each model within the ensemble
        \State $R=0$
        \ForEach{$k=1..K \text{ models}$}%
            \ForEach{$p=1..P \text{ particles}$}%
                \State $s_{1} = s_{T}$
                \ForEach{$t=1..H \text{ horizon}$}%
                % \State $\epsilon\sim \mathcal{N}(0,\sigma)$\algorithmiccomment{Sample noise}
                    \State $a_t = \textbf{A}_{j,t}$
                    \State $s_{t+1} =\hat{M}_\psi(s_t,a_t)$
                    \State $R = R + \gamma^{t-1}(\mathbbm{1}(t=H) Q_\theta(s_t,a_t)+\mathbbm{1}(t\neq H) \hat{r}(s_t,a_t))$
                    \State $C = C + \gamma^{t-1}(\hat{c}(s_t,a_t))$
                \EndForEach
            \EndForEach
        \EndForEach
         \State $\textbf{R}_j = \frac{1}{K} \sum_{k=1}^K(R/P)$
         \algorithmiccomment{Average Reward across the ensemble}
         \State $\textbf{C}_j =  \max_{[K]}\max_{[P]}(C)$  \algorithmiccomment{Maximum Cost across the ensemble and particles}
    \EndForEach
    % \State $A_{\text{elites}} =    \textbf{A}_N[\text{argsort}(\textbf{R}_N)[-e:] ]$ 

    \If {$\text{count}(\textbf{C}_{1:N}<d_0)<m$}
        
        \State $\mu_{new}$ = \text{weighted-mean}($\textbf{A}_{1:N}$, weights = exp($- \textbf{C}_N/\eta$))
        \State $\Sigma_{new}$ = \text{weighted-mean}($(\textbf{A}_{1:N}-\mu_{new})^2$, weights = exp($- \textbf{C}_N/\eta$)) 
        \State \algorithmiccomment{Weighted mean w.r.t neg-cost}
        \Else
        \State $\text{safe-idx}=\{i \text{ for }\textbf{C}_i<d_0\} $
        \State  $\mu_{new}$= \text{weighted-mean}($\textbf{A}_{\text{safe-idx}}$, weights = exp($ \textbf{R}_{\text{safe-idx}/\eta}$))
        \State $\Sigma_{new}$ = \text{weighted-mean}($\textbf{A}_{\text{safe-idx}}-\mu_{new})^2$, weights = exp($ \textbf{R}_{\text{safe-idx}}/\eta$))
       \State \algorithmiccomment{Weighted mean w.r.t safe actions}
    \EndIf
    \State $\mu^T_{i+1} =  \ \alpha * \mu_{new} + (1-\alpha)\mu^T_{i}$ \algorithmiccomment{Update mean}
    \State $\Sigma^T_{i+1} =  \ \alpha * \Sigma_{new} + (1-\alpha)\Sigma^T_{i}$ \algorithmiccomment{Update variance}
    \EndForEach
  \end{algorithmic}
  Output: $\mu^T=\mu^T_{n_{ARC}+1}$
  \label{alg:safeARC}
\end{algorithm}

\subsection{Discussion on the choice of terminal value function}

\reb{LOOP-SAC, LOOP-SARSA and POLO use different ways to learn a terminal value funcion. LOOP-SARSA is evaluating the "replay buffer policy" instead of the H-step lookahead policy because we are using off-policy data (where the original SARSA is an on-policy algorithm). We believe this is the main reason behind its poor performance. Unfortunately, on-policy LOOP-SARSA would be too slow, due to the need for collecting on-policy data.  POLO is formulated to evaluate $V^\pi$ with the model. However, POLO requires running trajectory optimization during the value function update, which is computationally expensive. In contrast to these methods, LOOP uses an off-policy algorithm to learn  $V^*$. We found that this approach has good performance and it is significantly more computationally efficient than POLO. An interesting direction of future work could be to try to combine LOOP with an efficient off-policy evaluation algorithm to estimate $V^\pi$. }

\newpage
\section{Experiment Details}
\label{ap:exp_details}
 We use the same hyperparameters for the underlying off-policy method (SAC) and the ensemble dynamics models following previous work for LOOP and all the baselines~\cite{nagabandi2019deep,chua2018deep,haarnoja2018soft}. All the results presented are averaged over 5 random seeds.

\subsection{Implementation Details for the Dynamics Model Ensemble}
Following \cite{nagabandi2019deep,chua2018deep}, we use probabilistic ensembles of dynamics models that capture the epistemic uncertainty as well as the aleatoric uncertainty in forward predictions~\cite{lakshminarayanan2016simple}. The dynamics model $\hat{M}$ is comprised of $K$ neural networks. Each individual network is randomly initialized and trained with the same dataset. Using the transition dataset, we train the dynamics model to predict the next state as well as the reward. In practice, instead of directly regressing to the next state, we instead predict $\Delta_{t+1}$, where $\Delta_{t+1}=s_{t+1}-s_{t}$ parametrized as a Gaussian distribution with a diagonal covariance matrix. We regress directly to the scalar reward. 

\subsection{Online RL}
\label{ap:online_rl}
\textbf{Additional details on PenGoal-v1 and Claw-v1: }
We modify the original Pen-v1 environment~\footnote{https://github.com/vikashplus/mj_envs} to have a narrower range of goals given by: $[0.7,0.7]+\mathcal{N}(0,0.1)$ and name this environment as PenGoal-v1. We use the Claw-v1 environment from~\citet{nagabandi2019deep} using the original implementation\footnote{https://github.com/google-research/pddm/tree/master/pddm} but we find the scale of rewards to be different from the paper.

\textbf{Baselines: } We use the original implementation  for MBPO\footnote{https://github.com/JannerM/mbpo}. For SAC, we use a public implementation \footnote{https://github.com/openai/spinningup}. We use a planning horizon of 3 for PETS-restricted which is the same as LOOP. LOOP-SARSA is based on the same H-step lookahead idea, but with a terminal value function that is a evaluation of the replay buffer. The value function is updated using the following SARSA update from the replay buffer transitions:
\begin{equation}
    \mathcal{T}^{\pi_D} Q(s_t,a_t) = r(s_t,a_t)+\gamma Q(s_{t+1},a_{t+1})~,\text{where}~(s_t,a_t,r_t,s_{t+1},a_{t+1}\sim D)
\end{equation}
This baseline is similar to MBOP~\cite{argenson2020model}. The main difference is that in this case the Q-function is learned via TD-backups for evaluation whereas MBOP uses Monte Carlo Evaluation. For SAC-VE, we implement H-step value expansion from~\cite{feinberg2018model} on top of SAC for a fair comparison. This is following the value expansions baseline implemented in MBPO~\cite{janner2019trust}.

\textbf{Training Details: }For LOOP-SAC we use SAC~\cite{haarnoja2018soft} as the underlying off-policy RL algorithm. Both the policy network (the parameterized actor) and the Q-function are parameterized by (256, 256) MLP with ReLU activations. The output of the policy network is a tanh squashed Gaussian. We use Adam to optimize both the policy and the Q-network with a learning rate of 3e-4. The temperature for SAC is learned to match a predefined target entropy. The replay buffer has a size of 1e6 and we use a batch size of 256. The target networks are updated with polyak averaging. Dynamics model related hyperparameters are listed in Table~\ref{tab:loop-sac-hp} and ARC related hyperparameters are in Table~\ref{tab:cem_hyperparams}. 

\vspace{5mm}
\begin{table}[h!]
  \begin{center}
    \begin{tabular}{l|c}
      \toprule % <-- Toprule here
      \textbf{Hyperparameter} & \textbf{Value}\\
      \midrule % <-- Midrule here
      Model Update frequency ($K_m$) & 250\\
      Ensemble Size ($K$) & 5\\
      Network Architecture & (200,200,200,200)\\
      Model Learning rate &  0.001\\
      \bottomrule % <-- Bottomrule here
    \end{tabular}
  \end{center}
  \caption{Dynamics Model Hyperparameters}
      \label{tab:loop-sac-hp}
\end{table}

\begin{table}[h!]
  \begin{center}
    
    \begin{tabular}{l|c}
      \toprule % <-- Toprule here
      \textbf{Hyperparamater} & \textbf{Value}\\
      \midrule % <-- Midrule here
      Planning Horizon ($H$) & 3\\
      Population Size ($N$) & 100\\
      Number of Particles ($P$) & 4\\
      Alpha ($\alpha$) & 0.1\\
      Iterations ($n_{ARC}$) & 5\\
      Mixture ratio ($\beta$) & 0.05\\
      Eta ($\eta$) & 1\\
      \bottomrule % <-- Bottomrule here
    \end{tabular}
  \end{center}
  \caption{Online RL: ARC Hyperparameters}
  \label{tab:cem_hyperparams}
\end{table}

\subsection{Offline RL}

\textbf{Baselines:} We reimplement the CRR baseline in Pytorch. For PLAS, we use the original implementation~\footnote{https://github.com/Wenxuan-Zhou/PLAS}. Note that LOOP requires terminal Q-functions which estimate the cumulative value of future rewards. Some offline RL methods such as CQL will not be suitable to be combined with LOOP because CQL estimates a conservative lower-bound of the Q-function~\cite{kumar2020conservative}. For MBOP~\cite{argenson2020model}, we report the results from their paper. 

\textbf{Training Details: }For both LOOP-CRR and LOOP-PLAS we use the provided hyperparameters in the original papers. To optimize for the H-step lookahead objective given in Eqn.~\ref{eq:H-step_objective}, we use ARC with 1 iteration of Iterative importance sampling and $\beta=1$ in the mixture prior. This is done to ensure that ARC trajectories are close to the actor trajectory distribution since the estimated Q-functions are only accurate within the data distribution. For each dataset, We perform an hyperparameter search over horizons $h$ - [2,4,10], pessimism parameter $\beta_{pess}$ - [0,0.5,1,5], exponential weighting temperature $1/\eta$ - [0.01,0.03,0.1,1,3,10] and noise standard deviation $\sigma$- [0.01,0.05,0.1]. We list the hyperparameters from the best experiments in Table~\ref{table:offline_rl_results}.

\begin{table*}[h]
    \vspace{2mm}
    \centering
    \footnotesize
    % \resizebox{\textwidth}{!}{%
    \begin{tabular}{c|c|cccc|cccc}
    \toprule
    \multirow{2}{*}{Dataset-type} &\multirow{2}{*}{Environments}  & \multicolumn{4}{c}{LOOP-CRR}& \multicolumn{4}{c}{LOOP-PLAS} \\
    & & $h$ & $1/\eta$ & $\beta_{pess}$  & $\sigma$& $h$ & $1/\eta$ & $\beta_{pess}$  & $\sigma$\\
    \midrule
    \multirow{3}{*}{random}& hopper&2&10.0&0.5&0.4&4&3&0.5&0.4\\
    & halfcheetah&2&0.01&5.0&0.4&2&1&0&0.01\\
    & walker2d&4&3.0&0.0&0.05&10&10&0&0.05\\
    \midrule
    \multirow{3}{*}{medium}& hopper&2&3.0&1.0&0.01&2&10&0.5&0.01\\
    & halfcheetah&2&0.01&0.0&0.01&2&0.01&1.0&0.01\\
    & walker2d&2&0.1&0.5&0.05&4&3.0&0&0.01\\
    \midrule
    \multirow{3}{*}{med-replay}& hopper&4&0.03&0.5&0.01&2&3.0&1.0&0.01\\
    & halfcheetah&2&1.0&0.5&0.01&2&0.03&5.0&0.01\\
    & walker2d&4&1.0&0.5&0.1&4&0.01&0.5&0.01\\
    \midrule
    \multirow{3}{*}{med-expert}& hopper &4&0.1&1.0&0.05&4&0.1&5.0&0.01\\
    & halfcheetah&2&0.01&5.0&0.01&2&0.01&0.5&0.05\\
    & walker2d&4&0.01&1.0&0.01&2&10.0&1.0&0.01\\
    \bottomrule
    \end{tabular}
\caption{Hyperparameters used in LOOP behavior policy during evaluation for Offline RL methods } 

%We highlight all the datasets where LOOP improves the performance by more than 5 percent. \dave{Bold all results in each row that are equivalently good}}
\label{table:offline_rl_ag_results}
\end{table*}

\subsection{Safe RL}
\label{ap:safe_rl}
\textbf{Details on the Environments: } For benchmarking safety environments we use the OpenAI safety gym environments~\cite{Ray2019}. We use a modified observation space for the agents where each agent observes its readings from velocimeter, magnetometer, and gyro sensors, LiDAR observations for the obstacles, and the goal location to a total of a 26-dimensional observation space. We also use an RC-car environment~\cite{ahn2019towards} in safe RL experiments shown in Figure~\ref{fig:safety_environments}. RC-car environment has a 6-dimensional observation space consisting of car's position and  rate of change of its position. It's action space comprises of throttle and steer command.
\begin{figure}[h]
\centering
\includegraphics[width=0.8\textwidth]{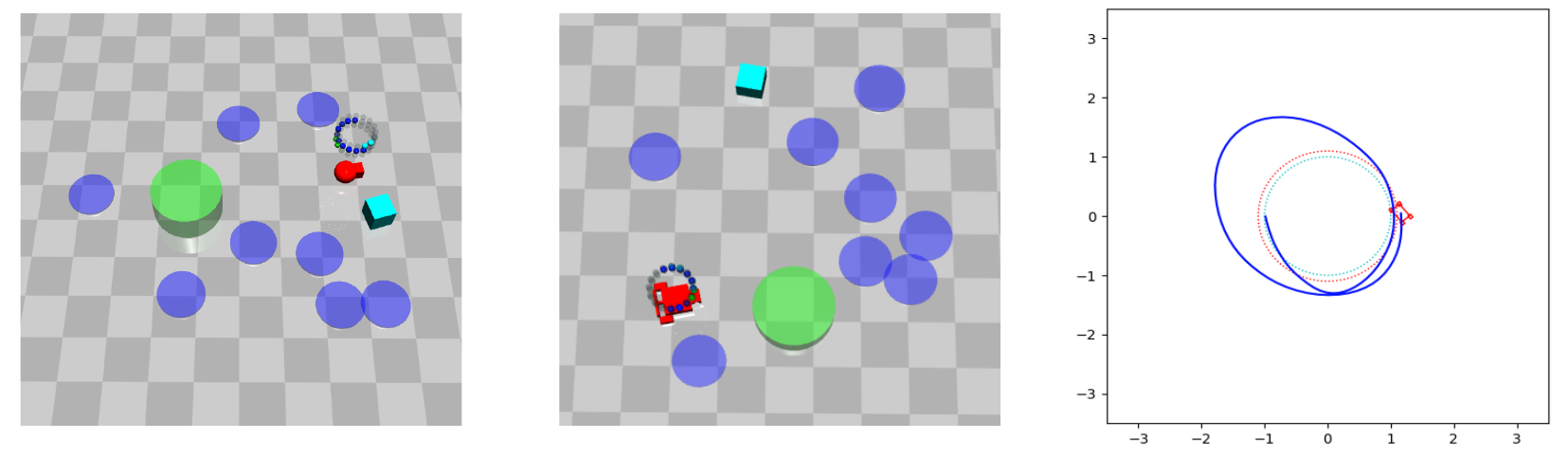}
\caption{Safety environments. Left to Right: PointGoal1, CarGoal1, Drift-v0}
\label{fig:safety_environments}
\end{figure}

\textbf{Baselines: } We compare against CPO~\cite{achiam2017constrained}, LBPO~\cite{sikchi2021lyapunov}, and PPO-lagrangian~\cite{altman1996constrained}. We use the original implementation for LBPO\footnote{https://github.com/hari-sikchi/LBPO} and the safety benchmark~\cite{Ray2019} for CPO and PPO-lagrangian. 
%CPO and LBPO are projection-style methods that ensures that the policy stays in the feasible set. They also have recovery behaviors to ensure that the policy becomes safe if the current policy is unsafe. PPO-lagrangian is a dual gradient descent method that does not enforces safety at each policy update, instead achieving a safe policy only asymptotically. 
All of the three baselines require a threshold to be set in order to optimize for safety. 
SafeLOOP optimizes for in-horizon safety whereas the baselines optimize for the infinite-step cumulative discounted return, so it becomes difficult to compare the methods directly. 
%We aim to show in our experiments that safeLOOP can achieve reasonable safety performance with high gains in sample efficiency. 
We design safeLOOP to optimize for 0 cost within the planning horizon and use the asymptotic safety cost reached by safeLOOP as the threshold for the baselines. We see that safeLOOP can reach average infinite horizon cost less than 10 which is lower than the threshold of 25 used in the official benchmark. 

% Second, we see in Figure~\ref{fig:openai_safety} that, comparing to the time to reach the same safety threshold as the baselines, safeLOOP is orders of magnitude more sample efficient and learns a safe policy more quickly. 

\textbf{Training details: } We use the safeARC algorithm presented in Algorithm~\ref{alg:safeARC} to solve the constrained optimization objective in Eqn.~\ref{eq:safe_loop}. The ARC parameters are the same as given in Table~\ref{tab:cem_hyperparams} with the Iterations(N) changed to 8 and Planning horizon(H) changed to 8. For OpenAI safety environments we use an action repeat of 5 across our method and the baselines.

\newpage
\section{Additional Experiments}
\label{ap:additional_experiments}

\subsection{Online RL experiments for additional environments}
\label{ap:online_rl_exp}
Figure~\ref{fig:loop_additional_main} shows the comparison of LOOP-SAC with baselines on additional tasks InvertedPendulum-v2, Swimmer, Hopper-v2 and TruncatedHumanoid-v2. LOOP-SAC is significantly more sample efficient  than SAC as we observed in Figure~\ref{fig:loop_main}. PETS-restricted has poor performance due to planning over a limited horizon. LOOP-SAC outperforms SAC-VE and is competitive to MBPO, except in Humanoid-v2 where MBPO outperforms. LOOP-SARSA has a poor performance across environments.

\vspace{3mm}
\begin{figure}[h]
\begin{center}
    \includegraphics[width=1.0\linewidth]{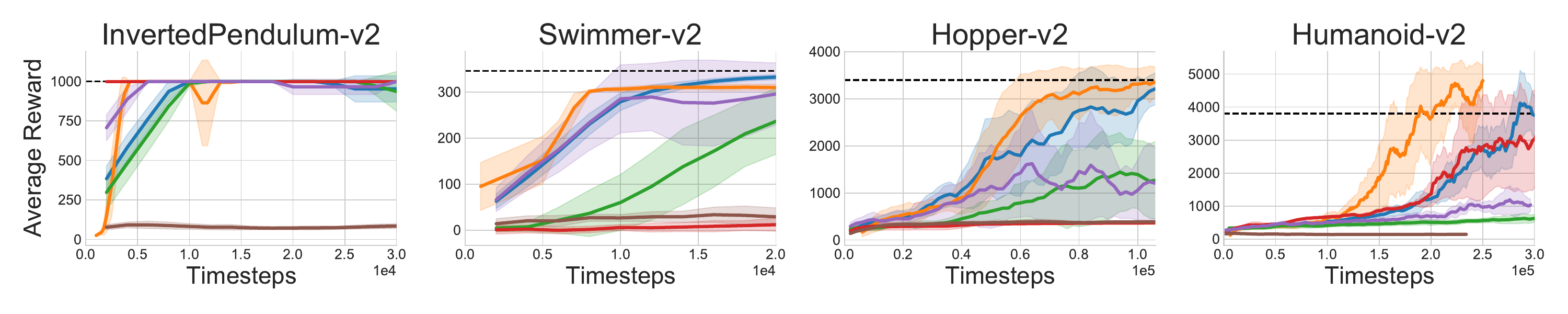}
      \\ 
    \includegraphics[width=1.0\linewidth]{figures/loop_main/loop4_pdf_legend.pdf}  
    \vspace{-7.2mm}
\end{center}
\caption{Comparisons of LOOP and the baselines for online RL for InvertedPendulum-v2, Swimmer, Hopper-v2 and TruncatedHumanoid-v2. LOOP-SAC is significantly more sample efficient than SAC. The dashed line indicates the performance of SAC at 1 million timesteps. }
\label{fig:loop_additional_main}
\end{figure}
\vspace{5mm}

\subsection{Comparison to modified POLO}
\label{ap:polo_experiments}

\reb{In this section, we compare against POLO for Claw-v0 and HalfCheetah-v0 in Figure 9. The author's implementation of POLO is unavailable so we tried our best to implement it. To have a fair comparison with LOOP, we keep the hyperparameters as close to LOOP as possible and use a learned model; since the code for POLO is not available, we are unsure what hyperparameters were used in the original experiments. The performance of POLO is pretty low compared to LOOP, potentially due to the limited computation of CEM used for the value function update. Potentially the performance of POLO would be better with much larger computational resources than what we have available.}

\reb{We would also like to highlight the difference in computational efficiency in LOOP and POLO. POLO requires an additional trajectory optimization procedure for value function computation, which is very computationally expensive. In contrast, LOOP learns a parameterized policy and value function to make the value function computation significantly faster.  In addition, normally for online planning, we can warm start from the results from the previous time step (``amortization"); LOOP and PETS~\cite{chua2018deep} take advantage of this optimization.  In contrast, POLO cannot take advantage of this amortization during optimization for the value computation because we sample states IID from the replay buffer. Our implementation of POLO (after reasonable optimizations) takes $\approx$84 hours for 100k steps of HalfCheetah on a single NVIDIA 1080 GPU whereas LOOP takes $\approx$7 hours (12x less computation) while taking $\approx$1/5 the memory consumption of POLO.}

\begin{figure}[h]
\begin{center}
      \includegraphics[width=0.6\linewidth]{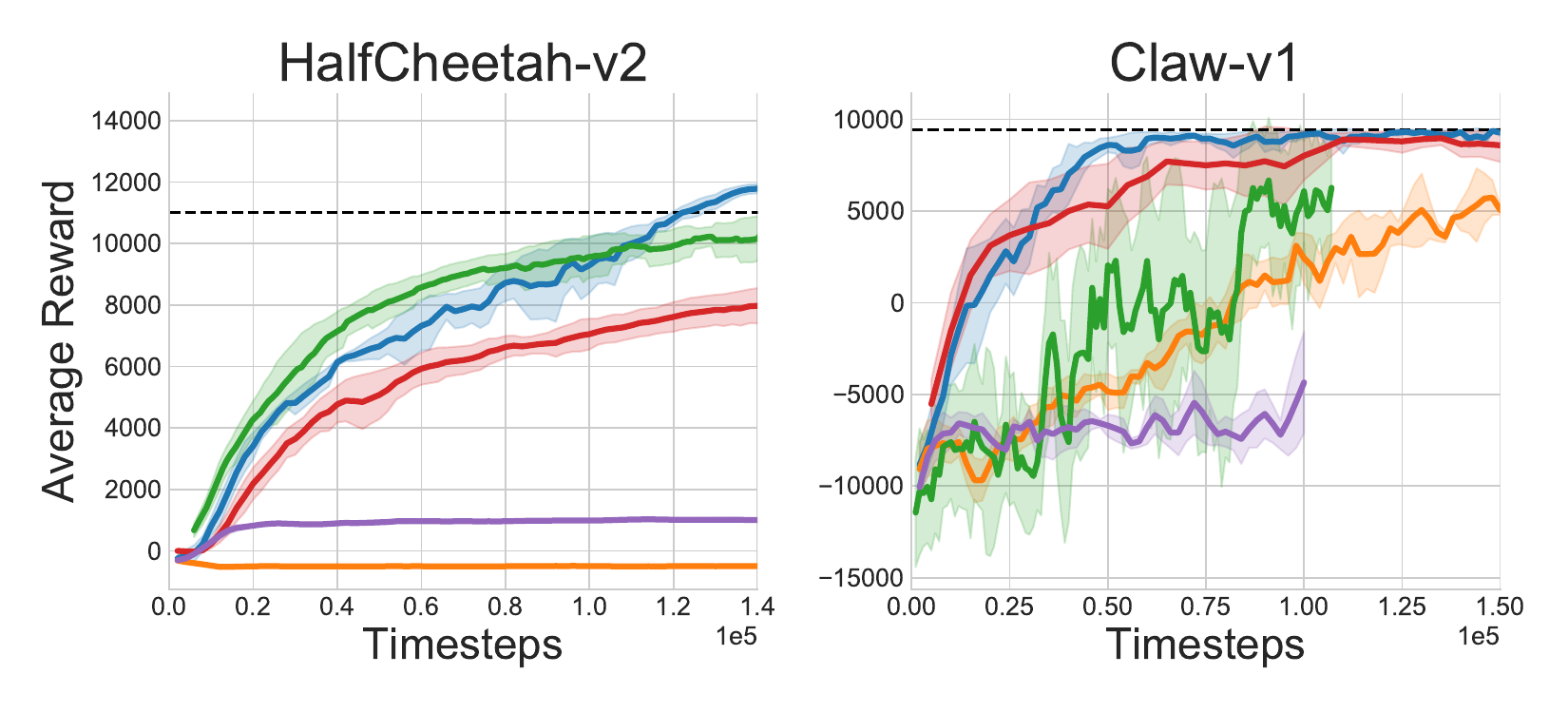}
      \\ 
    \includegraphics[width=1.0\linewidth]{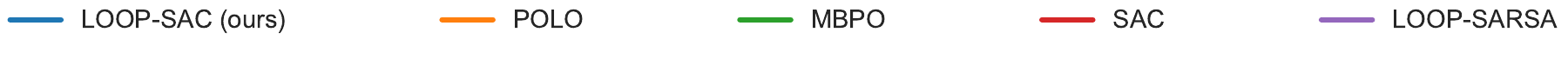}  
    \vspace{-7.2mm}
\end{center}
\caption{Comparisons of LOOP with modified-POLO for online RL for HalfCheetah-v2 and Claw-v1. POLO demonstrates poor performance which might be attributed to one of the reasons mentioned above. }
\label{fig:polo_additional_main}
\end{figure}

\subsection{Offline RL experiments for D4RL}
\label{ap:offline_rl_exp}
Table~\ref{table:complete_offline_rl_results} shows the performance of LOOP on four types of D4RL 
locomotion datasets. The random dataset is generated by a randomly initialized policy. The medium dataset is generated by executing a ``medium quality'' policy trained up to half of the final performance at convergence. The medium-replay dataset is the replay buffer of the medium quality policy. The medium-expert dataset is generated by a medium quality policy and a fully trained policy.

\begin{table*}[h]
    \vspace{2mm}
    \centering
    \footnotesize
    % \resizebox{\textwidth}{!}{%
    \begin{tabular}{c|c|c|c|c|c|c|c|c}
    \toprule
    \multirow{2}{*}{Dataset} & Env & CRR & LOOP & Improve\% & PLAS & LOOP & Improve\%& MBOP \\
    &&&CRR&&&PLAS&\\
    \midrule
    \multirow{3}{*}{random}& hopper&10.40&10.68&2.7&10.35&10.71&3.5&\textbf{10.8}\\
    & halfcheetah&4.23&7.55&78.5&26.05&\textbf{26.14}&0.3&6.3\\
    & walker2d&1.94&2.04&5.2&0.89&2.83&218.0&\textbf{8.1}\\
    \midrule
    \multirow{3}{*}{medium}& hopper&65.73&\textbf{85.83}&30.6&32.08&56.47&76.0&48.8\\
    & halfcheetah&41.14&41.54&1.0&39.33&39.54&0.5&\textbf{44.6}\\
    & walker2d&69.98&\textbf{79.18}&13.1&46.20&52.66&14.0& 41.0\\
    \midrule
    \multirow{3}{*}{med-replay}& hopper&27.69&29.08&5.0&29.29&\textbf{31.29}&6.8& 12.4\\
    & halfcheetah&42.29&42.84&1.3&43.96&\textbf{44.25}&0.7&42.3\\
    & walker2d&19.84&27.30&37.6&35.59&\textbf{41.16}&15.7&9.7\\
    \midrule
    \multirow{3}{*}{med-expert}& hopper&112.02&113.71&1.5&110.95&\textbf{114.32}&3.0&55.1\\
    & halfcheetah&21.48&24.19&12.6&93.08&98.16&5.5&\textbf{105.9}\\
    & walker2d&103.77&\textbf{105.76}&1.9&90.07&99.03&9.9&70.2\\
    \bottomrule
    \end{tabular}
\caption{Normalized scores for LOOP on the D4RL datasets comparing to the underlying offline RL algorithms and a baseline MBOP. LOOP improves the base algorithm across various types of datasets and environments. } 
\label{table:complete_offline_rl_results}
\end{table*}

\reb{\subsection{Pessimism ablation for Offline RL}}

\reb{Table~\ref{table:pessimism_offline_rl_results} shows an ablation of the pessimism term $\beta_{pess}$ in Eqn.~\ref{eq:offline_loop} as used in LOOP for Offline RL experiments. We note that the pessimistic term is not itself one of our contributions; this pessimistic term was used in previous works in model-based offline RL like~\cite{yu2020mopo,kidambi2020morel} which learn a policy given the data in an uncertainty penalized MDP. We observe that being pessimistic allows us to control incorrect extrapolation and obtain higher returns in most of the environments. }

\begin{table*}[h]
    \vspace{2mm}
    \centering
    \footnotesize
    % \resizebox{\textwidth}{!}{%
    \begin{tabular}{c|c|c|c|c|c|c|c}
    \toprule
    \multirow{3}{*}{Dataset} & Env & LOOP & LOOP & $\beta^*$ & LOOP & LOOP & $\beta^*$ \\
    &&CRR&CRR&&PLAS&PLAS&\\
    &&($\beta=0$)&($\beta=\beta^*$)&&($\beta=0$)&($\beta=\beta^*$)&\\
    \midrule
    \multirow{3}{*}{random}& hopper&10.31&\textbf{10.68}&0.5&10.67&\textbf{10.71}&0.5\\
    & halfcheetah&5.12&\textbf{7.55}&5.0&\textbf{26.14}&26.14&0.0\\
    & walker2d&\textbf{2.04}&2.04&0.0&\textbf{2.83}&2.83&0.0\\
    \midrule
    \multirow{3}{*}{medium}& hopper&78.56&\textbf{85.83}&1.0&54.97&\textbf{56.47}&0.5\\
    & halfcheetah&\textbf{41.54}&41.54&0.0&38.01&\textbf{39.54}&1.0\\
    & walker2d&75.21&\textbf{79.18}&0.5&\textbf{52.66}&52.66&0.0\\
    \midrule
    \multirow{3}{*}{med-replay}& hopper&28.28&\textbf{29.08}&0.5&31.08&\textbf{31.29}&1.0\\
    & halfcheetah&42.71&\textbf{42.84}&0.5&44.01&\textbf{44.25}&5.0\\
    & walker2d&23.17&\textbf{27.30}&0.5&32.99&\textbf{41.16}&0.5\\
    \midrule
    \multirow{3}{*}{med-expert}& hopper&104.57&\textbf{113.71}&1.0&98.87&\textbf{114.32}&5.0\\
    & halfcheetah&23.84&\textbf{24.19}&5.0&94.19&\textbf{98.16}&0.5\\
    & walker2d&104.57&\textbf{105.76}&1.0&97.87&\textbf{99.03}&1.0\\
    \bottomrule
    \end{tabular}
\caption{Normalized scores for LOOP on the D4RL datasets ablating the pessimism parameter.} 
\label{table:pessimism_offline_rl_results}
\end{table*}

\subsection{Empirical analysis for ARC}
\label{ap:arc_experiments}

In this section, we aim to verify how the ARC and its specfic hyperparameters affect the performance of LOOP for both the online RL and offline RL settings.

\subsubsection{Ablation Study on Actor Regularized Control}

In this experiment, we compare the performance of LOOP with ARC to a variant of LOOP which optimize Eqn.~\ref{eq:H-step_objective} without any constraint using CEM in the Online RL setting. CEM starts the optimization from the mean action sequence from the previous environment time step. It does not include actions proposed by the parameterized actor in the population. During training, we measure the actor-divergence defined to be the $L_2$ distance between the proposed action means of the parameterized actor and the CEM output.

The results are shown in Figure~\ref{fig:loop_ag_performance}. The training process sometimes become unstable in the absence of ARC. We also observe that ARC empirically reduces actor-divergence during training. 

\begin{figure}[h]
\begin{center}
      \includegraphics[width=1.0\linewidth]{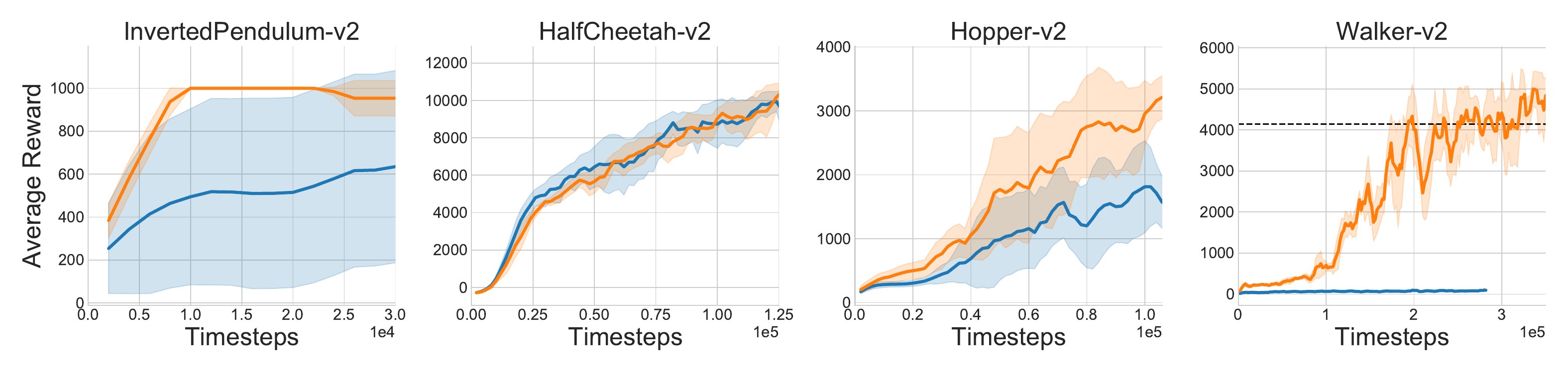}\\ 
    \includegraphics[width=1.0\linewidth]{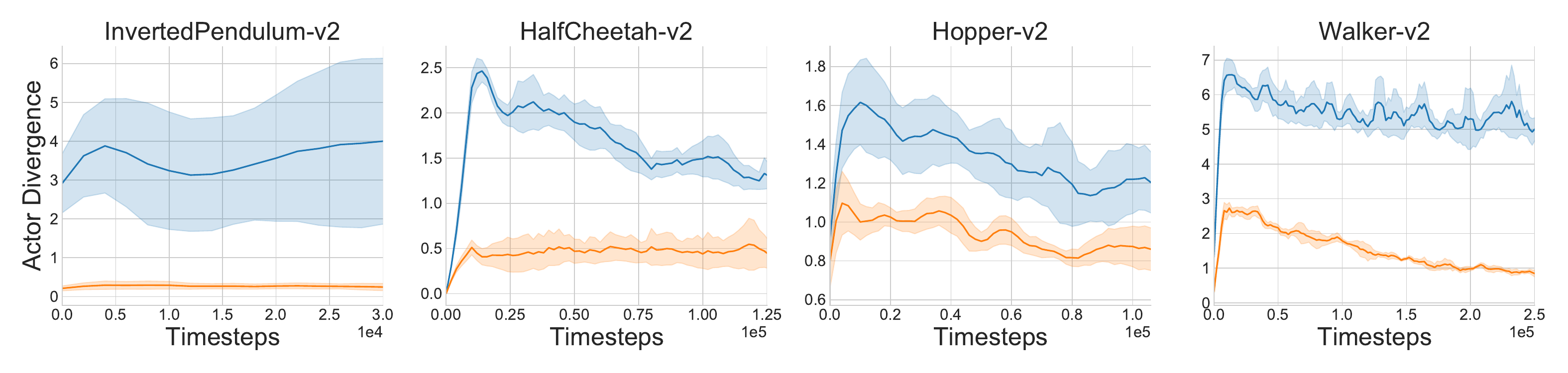}\\ 
    \includegraphics[width=0.6\linewidth]{figures/ablation_main_paper/ag_main_pdf_legend.pdf}
\end{center}
\caption{Top: We illustrate the effect of ARC on the performance for online-RL. Without ARC, the performance is worse and the training becomes unstable. Bottom: We illustrate that ARC effectively reduces the actor-divergence between the H-step lookahead policy and the parameterized actor.}
\label{fig:loop_ag_performance}
\end{figure}
\vspace{0.5cm}

\subsubsection{ARC runtime}

ARC runs at 14.3 Hz for the HalfCheetah-v1 environment with the hyperparameters specified in Table~\ref{tab:cem_hyperparams} on a machine with Intel(R) Xeon(R) CPU E5-2620 v4 @ 2.10GHz and NVIDIA GeForce GTX 1080 Ti with a GPU memory consumption of ~1500 MB.

\subsubsection{ Effect of $\beta$ in ARC for Online RL}

We see in Section~\ref{sec:actor_divergence} that ARC uses a mixture distribution of actor and 1-step shifted output from the previous timestep as the prior given by: 
\begin{equation*}
    p^\tau_{prior}=\beta \pi_\theta+(1-\beta)\mathcal{N}(\mu_{t-1},\sigma)
\end{equation*}
where $\beta$ is the mixture coefficient.

In this experiment, we compare ARC with different parameters of beta for online RL. We observe empirically in Figure~\ref{fig:loop_ag_ar} that ARC with $\beta<1$ is more suitable to online RL as it is less restrictive and allows for a greater improvement on the parametric actor. 
\begin{figure}[h]
\begin{center}
      \includegraphics[width=0.8\linewidth]{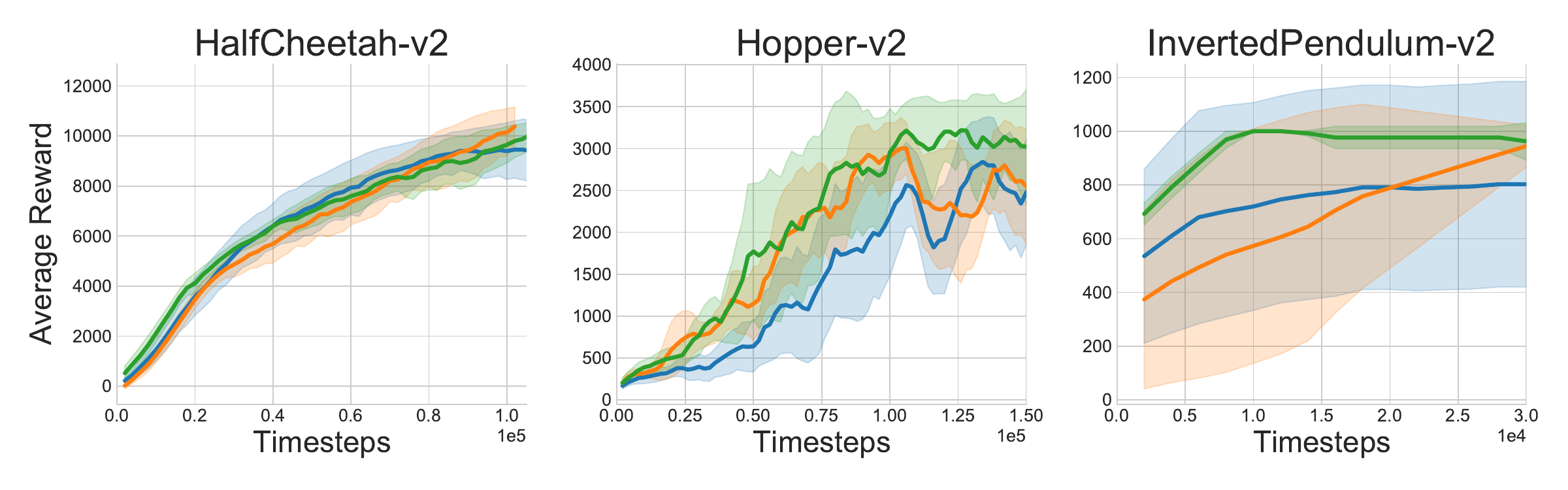}\\ 
    \includegraphics[width=0.8\linewidth]{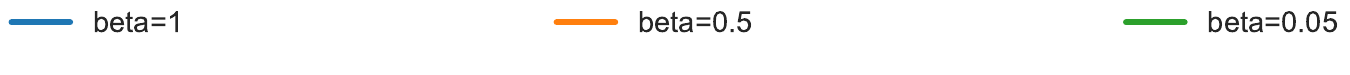}  
\end{center}
\caption{Effect of $\beta$ in Online RL experiments. A high $\beta$ constrains output actions to be close to actor and can restrict policy improvement.}
\label{fig:loop_ag_ar}
\end{figure}

\subsubsection{Effect of $\beta$ in ARC for in Offline RL}

 We use ARC with 1 importance sampling iteration for offline RL. In the following experiment, we compare ARC over 1 importance sampling iteration with $\beta=1$ against ARC over 5 iterations with $\beta=0.05$. The latter version utilizes the solution obtained by ARC in previous timestep which may potentially select trajectories that lead to out-of-distribution states with overestimated value. We see in Table~\ref{tab:loop_offline_ag} that offline RL results for LOOP using ARC with $\beta=0.05$ has much worse performance than ARC with $\beta=1$ and also performs worse than the underlying offline RL method.

\begin{table}[h]
    \vspace{2mm}
    \centering
    \footnotesize
    % \resizebox{\textwidth}{!}{%
    \begin{tabular}{c|c|c|c|c|c|c|c}
    \toprule
    \multirow{2}{*}{Dataset} & Env & CRR & LOOP-CRR & LOOP-CRR  & PLAS & LOOP-PLAS & LOOP-PLAS \\
    &&&($\beta$=0.05)&($\beta$=1.0)&&($\beta$=0.05)&($\beta$=1.0)\\
    \midrule
    \multirow{3}{*}{random}& hopper&10.40&7.50&10.68&10.35&7.66&10.71\\
    & halfcheetah&4.23&2.32&7.55&26.05&5.21&26.14\\
    & walker2d&1.94&2.43&2.04&0.89&1.24&2.83\\
    \midrule
    \multirow{3}{*}{medium}& hopper&65.73&15.02&85.83&32.08&9.64&56.47\\
    & halfcheetah&41.14&3.09&41.54&39.33&2.99&39.54\\
    & walker2d&69.98&6.02&79.18&46.20&4.25&52.66\\
    \midrule
    \multirow{3}{*}{med-replay}& hopper&27.69&8.78&29.08&29.29&6.8&31.29\\
    & halfcheetah&42.29&3.10&42.84&43.96&4.68&44.25\\
    & walker2d&19.84&6.02&27.30&35.59&6.89&41.16\\
    \midrule
    \multirow{3}{*}{med-expert}& hopper&112.02&8.78&113.71&110.95&7.48&114.32\\
    & halfcheetah&21.48&3.09&24.19&93.08&4.10&98.16\\
    & walker2d&103.77&6.01&105.76&90.07&3.01&99.03\\
    \bottomrule
    \end{tabular}
    \vspace{2mm}
\caption{Effect of $\beta$ in  offline-LOOP for the offline RL experiments. A low $\beta$ can potentially select trajectories with overestimated returns.} 
\label{tab:loop_offline_ag}
\end{table}

\subsection{Benefits of deploying H-step lookahead in Online RL}

%\subsubsection{LOOP-MOPO vs MOPO online}
In LOOP we use a H-step lookahead policy for both exploration and evaluation. In this experiment, we run the Online RL experiments with LOOP only used for evaluation but not for exploration similar to the Offline RL experiments. This baseline is named SAC-evalLOOP. From Figure~\ref{fig:loop_sac_online}, LOOP-SAC outperforms both SAC-evalLOOP and SAC, which shows the benefits of H-step lookahead in LOOP during training-time deployment.

\begin{figure}[h]
\begin{center}
      \includegraphics[width=0.8\linewidth]{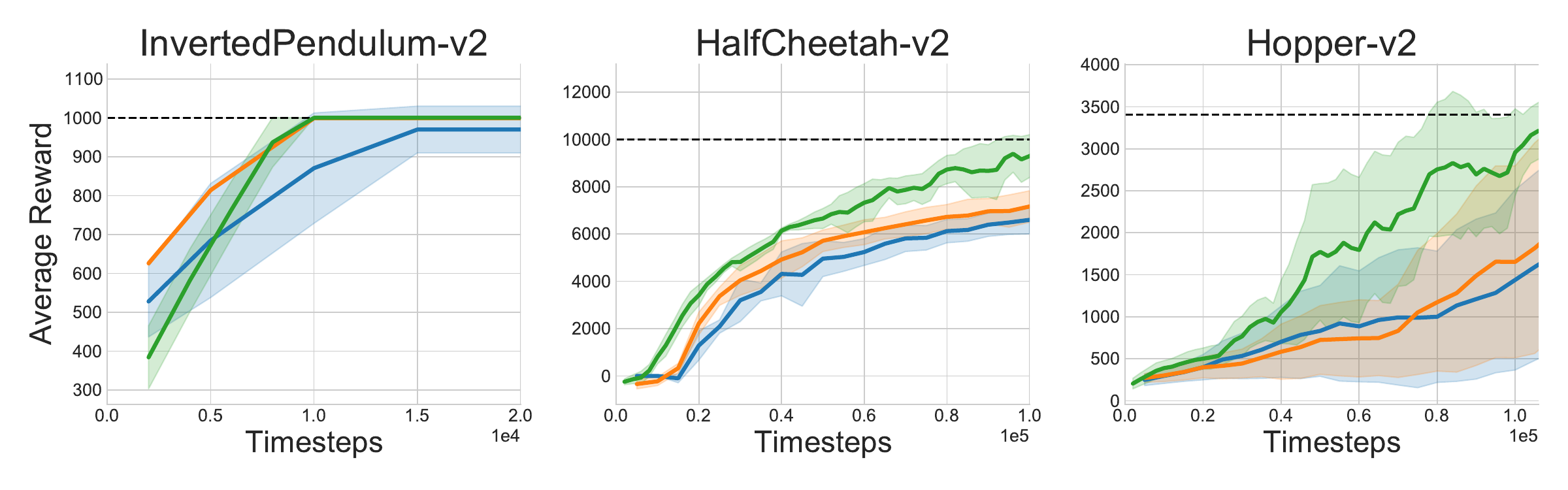}\\ 
    \includegraphics[width=0.8\linewidth]{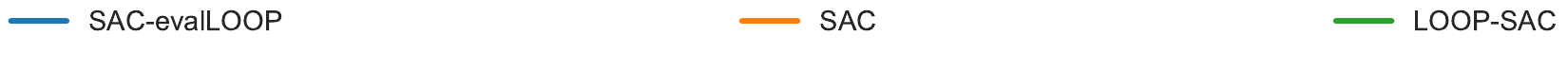}  
\end{center}
\caption{We compare the performance of LOOP-SAC to SAC-Online and SAC-Online-LOOPeval. LOOP-SAC outperforms both baselines and suggests that LOOP benefits from the H-step lookahead policy during training-time deployment.}
\label{fig:loop_sac_online}
\end{figure}

\subsection{Using Offline RL Algorithms with LOOP for Online RL}
\label{ap:offline_rl_in_loop}
In Section~\ref{sec:actor_divergence}, we mentioned that naively combining H-step lookahead policy with an off-policy algorithm will lead to the Actor Divergence issue. One potential solution we have considered besides ARC is to use an Offline RL algorithm as the underlying off-policy algorithm. Offline RL algorithm are designed to train a policy over a static dataset that is not collected by the parameterized actor which in principle should mitigate the instability issues of the value function learning caused by Actor Divergence. Note that in this case we are considering an Online RL \textit{problem setting} with the help of Offline RL \textit{algorithms}.

We investigate a combination of LOOP with Offline RL methods MOPO~\cite{yu2020mopo} and CRR~\cite{wang2020critic}. We reimplement MOPO in PyTorch (originally in Tensorflow) for compatibility with other modules of LOOP. We also modify the dynamics model activations from Swish to ReLU. We use the same CRR implementation as the Offline RL experiments discussed above.
We adapt MOPO and CRR to the Online RL setting by updating the policy and the value function for 20 gradient updates for each environment timestep. From Figure~\ref{fig:loop_offline}, LOOP-SAC has the most consistent performance across all the environments. LOOP-CRR and LOOP-MOPO work well in some cases but are significantly worse than LOOP-SAC in the others.

\vspace{5mm}
\begin{figure}[h]
\begin{center}
      \includegraphics[width=1.0\linewidth]{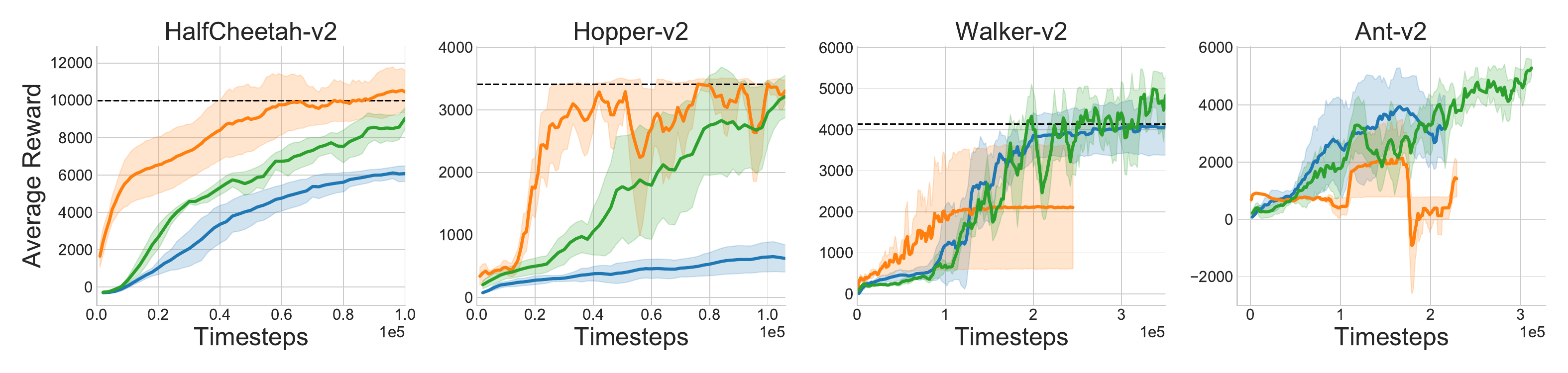}\\ 
    \includegraphics[width=0.8\linewidth]{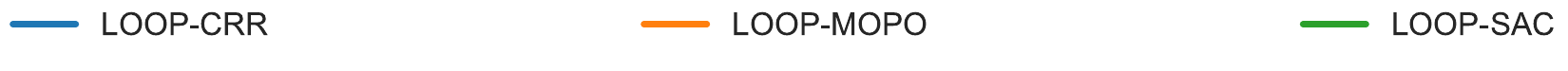}  
\end{center}
\caption{Using offline RL methods like CRR (model-free) or MOPO (model-based) with LOOP does not lead to consistently better performance.}
\label{fig:loop_offline}
\end{figure}

\vspace{5mm}
\begin{table}[h!]
  \begin{center}
    \begin{tabular}{|l|c|c|c|}
      \toprule % <-- Toprule here
      Model Update frequency ($K_m$) & \multicolumn{3}{c|}{250}\\
      \midrule
      Ensemble Size & \multicolumn{3}{c|}{5}\\
      \midrule
      Network Architecture & \multicolumn{3}{c|}{MLP with 4 hidden layers of size 200}\\
      \midrule
      Model Horizon (H) & \multicolumn{3}{c|}{3}\\
      \midrule
      Model Learning rate &  \multicolumn{3}{c|}{0.001}\\
      \midrule
      Policy update per environment step (R) &  \multicolumn{3}{c|}{20}\\
      \midrule
      Replay Buffer Size &  \multicolumn{3}{c|}{1e6}\\
      \midrule
      Gradient updates per timestep(R) & \multicolumn{3}{c|}{20}\\
      \midrule
      Pessimism parameter($\lambda$) & \multicolumn{3}{c|}{1}\\
      \midrule
      Model rollout length & \multicolumn{3}{c|}{1}\\
      \bottomrule % <-- Bottomrule here
    \end{tabular}
  \end{center}
  \caption{LOOP-MOPO Hyperparameters}
      \label{tab:loop-mopo-hp}
\end{table}

\end{document}